\documentstyle[jair,twoside,11pt,theapa]{article}
\input epsf 
\input psfig.sty

\jairheading{31}{2008}{431-472}{11/2007}{3/2008}
\firstpageno{431}

\ShortHeadings{First Order Decision Diagrams for Relational MDPs}
{Wang, Joshi, \& Khardon}

\newtheorem{procedure}{Procedure}
\newtheorem{theorem}{Theorem}
\newtheorem{claim}{Claim}
\newtheorem{definition}{Definition}

\newtheorem{lemma}{Lemma}

\newenvironment{proof}{{\em Proof:}}{\hfill{$\Box$}}

\newcommand{\map}{{\mbox{MAP}}}
\newcommand{\true}{\texttt{true }}
\newcommand{\false}{\texttt{false }}

\newcommand{\falseE}{\texttt{false}}

\newcommand{\cond}[1]{\par\noindent{\bf (#1)} }

\newcommand{\B}{{\cal B}}
\newcommand{\downt}{_{\downarrow t}}
\newcommand{\downf}{_{\downarrow f}}
\newcommand{\downa}{_{\downarrow a}}
\newcommand{\downb}{_{\downarrow b}}
\newcommand{\downna}{_{\downarrow \overline{a}}}
\newcommand{\downnb}{_{\downarrow \overline{b}}}
\newcommand{\NF}{{\mbox{NF}}}
\newcommand{\EF}{{\mbox{EF}}}
\newcommand{\rfiveremove}{{\mbox{R5}}}

\newcommand{\rsevenreplace}{{\mbox{R7-replace}}}
\newcommand{\rsevendrop}{{\mbox{R7-drop}}}
\newcommand{\reightunify}{{\mbox{R8}}}
\newcommand{\rnine}{{\mbox{R9}}}

\begin{document}

\title{First Order Decision Diagrams for Relational MDPs} 

\author{\name Chenggang Wang \email cwan@cs.tufts.edu \\
\name Saket Joshi \email sjoshi01@cs.tufts.edu \\
\name Roni Khardon \email roni@cs.tufts.edu \\
\addr  Department of Computer Science, Tufts University \\
                 161 College Avenue, Medford, MA 02155, USA
}

\maketitle

\begin{abstract} 
Markov decision processes capture sequential decision making under
uncertainty, where an agent must choose actions so as to optimize long
term reward.  The paper studies efficient reasoning mechanisms for
Relational Markov Decision Processes (RMDP) where world states have an
internal relational structure that can be naturally described in terms
of objects and relations among them.  Two contributions are presented.
First, the paper develops First Order Decision Diagrams (FODD), a new
compact representation for functions over relational structures,
together with a set of operators to combine FODDs, and novel reduction
techniques to keep the representation small.  Second, the paper shows
how FODDs can be used to develop solutions for RMDPs, where reasoning
is performed at the abstract level and the resulting optimal policy is
independent of domain size (number of objects) or instantiation.  In
particular, a variant of the value iteration algorithm is developed by
using special operations over FODDs, and the algorithm is shown to
converge to the optimal policy.
\end{abstract}

\section{Introduction}

\label{Cha:intro}
Many real-world problems can be cast as sequential decision making under uncertainty. 
Consider a simple example in a logistics domain where an agent delivers boxes. 
The agent can take three types of actions: to load a box on a truck,
to unload a box from a truck, and to drive a truck to a city.
However  the effects of actions may not be perfectly predictable.
For example its gripper may be slippery 
so load actions may not succeed, or its navigation module may not be reliable and it may
end up in a wrong location. This uncertainty compounds the already
complex problem of planning a course of action to achieve some goals
or maximize rewards. 

Markov Decision Processes (MDP)
 have become the standard model for sequential decision making 
under uncertainty \cite{BoutilierDeHa1999}.  These models also provide 
a general framework for artificial intelligence (AI) planning, 
where an agent has to achieve or maintain a well-defined goal.
MDPs model an agent interacting with the world. 
The agent
can fully observe the state of
the world and takes actions so as to change the state.
In doing that, the agent 
tries to optimize a measure of the long term reward it can obtain
 using such actions.

The classical representation and algorithms for MDPs \cite{Puterman1994}
require enumeration of the state space.
For more complex situations 
we can specify the state space in terms of a set of propositional variables 
called state attributes.
These state attributes together determine the world state. Consider a very simple
logistics problem that has only one box and one truck.
Then we can have state attributes such as 
truck in Paris (TP), box in Paris (BP), box in Boston (BB), etc. 
If we let the state space
 be represented by $n$ binary state attributes 
then the total number of states 
would be $2^n$. 
For some problems, however, the domain dynamics and resulting
solutions have a simple structure that can be described compactly
using the state attributes, and previous work 
known as the {\em propositionally factored approach}
has developed 
a suite of algorithms that take advantage of such structure and avoid 
state enumeration. 
For example, one can 
use dynamic Bayesian networks, decision trees, and algebraic decision diagrams
to concisely represent the MDP model. 
This line of work showed substantial
speedup for propositionally factored domains
\cite{BoutilierDeMo1995,BoutilierDeGo2000,HoeyStHuBo1999}. 

The logistics example presented above is very small. Any realistic
problem will have a large number of objects and corresponding
relations among them. 
Consider a problem with
four trucks, three boxes, and where the goal is to have
a box in Paris, but it does not matter which box is in Paris.
With the propositionally factored approach, 
we need to have one propositional variable
for every possible instantiation of the 
relations in the domain, e.g., box 1 in Paris, box 2 in Paris,
box 1 on truck 1, box 2 on truck 1, and so on,
and the action space expands in the same way.
The goal becomes a ground disjunction over different instances stating
``box 1 in Paris, or box 2 in Paris,
or box 3 in Paris, or box 4 in Paris''. 
Thus we get 
a very large MDP and at the same time we lose the 
structure implicit in the relations  and 
the potential benefits of this structure in terms of computation. 

This is the main motivation behind relational or first order MDPs 
(RMDP).\footnote{
\citeA{SannerBo2005} make a distinction between first order MDPs that
can utilize the full power of first order logic to describe a problem
and relational MDPs that are less expressive. We follow this in 
calling our language RMDP.
}
A first order representation of MDPs
can describe domain objects and relations among 
them, and can use quantification in specifying objectives. 
In the logistics example, we can introduce three predicates to capture the
relations among domain objects, i.e.,
$Bin (Box, City)$,
$Tin (Truck, City)$, and
$On (Box, Truck)$ with their
obvious meaning.
We have three parameterized actions, i.e.,
$load (Box, Truck)$, 
$unload (Box, Truck)$, and 
$drive (Truck, City)$.
Now domain dynamics, reward, and solutions can be described compactly
and abstractly using the relational notation.
For example,
we can define the goal using existential quantification,
i.e., $\exists b, Bin(b,Paris)$.
Using this goal one can identify an abstract policy,
which is optimal for every possible instance of the domain.
Intuitively when there are $0$ steps to go, 
the agent will be rewarded if there is any box in Paris.
When there is one step to go and there is no box in Paris yet, 
the agent can take one action to help achieve the goal.
If there is a box (say $b_1$) on a truck (say $t_1$) and the truck is in Paris, 
then the agent can execute the action $unload(b_1, t_1)$, which may make 
$Bin(b_1, Paris)$ true, thus the goal will be achieved.
When there are two steps to go,
if there is a box on a truck that is in Paris,
the agent can take the $unload$ action twice (to increase 
the probability of successful unloading of the box), or
if there is a box on a truck that is not in Paris,
the agent can first take the action $drive$ followed by $unload$.
The preferred plan will depend on the success probability 
of the different actions.
The goal of this paper is to develop efficient solutions for 
such problems using a relational approach,   
which
performs general reasoning in solving problems and does not
propositionalize the domain. 
As a result the complexity of our algorithms
does not change when the number of domain objects changes. 
Also the solutions obtained are good for any domain of any size 
(even infinite ones) simultaneously.  
Such an abstraction is not possible within the propositional approach.

Several approaches for 
solving RMDPs were developed over the last few years.
Much of this work was devoted to developing techniques to 
{\em approximate} RMDP solutions
using different representation languages and algorithms
\cite{GuestrinKoGeKa2003,FernYoGi2003,GrettonTh2004,SannerBo2005,SannerBo2006}.
For example, 
\citeA{DzeroskiDeDr01} and \citeA{Driessens2006} use reinforcement
learning techniques with relational representations. 
\citeA{FernYoGi2006} and \citeA{GrettonTh2004} use inductive 
learning methods to learn a value map or policy from solutions or 
simulations of small instances.
\citeA{SannerBo2005,SannerBo2006} develop an approach to 
approximate value iteration 
that 
does not need to propositionalize the domain.
They represent value functions
as a linear combination of first order basis functions and obtain the
weights by lifting the propositional approximate 
linear programming techniques \cite{schuurmansPa2001,GuestrinKoPaVe2003} to handle the first order case. 

There has also been work on exact solutions such as 
symbolic dynamic programming (SDP) \cite{BoutilierRePr2001},
the relational Bellman algorithm (ReBel) \cite{KerstingOtRa2004},
and first order value iteration (FOVIA) \cite{GroBmannHoSk2002,HolldoblerKaSk2006}.
There is no working
implementation of SDP because it is hard to keep the state formulas
consistent and of manageable size in the context of the situation
calculus. Compared with SDP, ReBel and FOVIA provide more practical
solutions. 
They both use restricted languages to represent RMDPs,
so that reasoning over formulas is easier to perform.
In this paper we develop a representation
that combines the strong points of these approaches.

Our work is inspired by the successful application of 
Algebraic Decision Diagrams (ADD)
\cite{Bryant1986,McMillan1993,BaharFrGaHaMaPaSo1993}
in solving propositionally factored MDPs and POMDPs 
\cite{HoeyStHuBo1999,St-AubinHoBo2000,HansenFe2000,FengHa2002}.
 The intuition behind this idea is that
the ADD representation allows information sharing, e.g., sharing 
the value of all states that belong to an ``abstract state'', so that algorithms can consider
many states together and do not need to resort to 
state enumeration. If there is sufficient regularity in the model,
ADDs can be very compact, allowing problems to be represented and
solved efficiently.
We provide a generalization of this approach by lifting ADDs 
to handle relational structure and adapting the MDP algorithms.
The main difficulty in lifting the propositional solution, 
is that in
relational domains the transition function specifies a set of
schemas for conditional probabilities. 
The 
propositional solution
uses the
concrete conditional probability to calculate the regression function. 
But this is not possible with schemas. 
One way around this problem is to first ground the domain and problem
at hand and only then perform the reasoning 
\cite<see for example>{SanghaiDoWe2005}. However this does not allow for solutions
abstracting over domains and problems.
Like SDP, ReBel, and FOVIA, our constructions do perform general
reasoning.

First order decision trees and even decision diagrams have already
been considered in the literature \cite{BlockeelDR98,GrooteTv2003}
and several semantics for such diagrams are possible. 
\citeA{BlockeelDR98} lift propositional decision trees 
to handle relational structure in the context of learning from relational datasets.
\citeA{GrooteTv2003} provide a notation for first order Binary Decision Diagrams (BDD)
that can capture formulas in Skolemized conjunctive normal form
and then provide a theorem proving algorithm based on this
representation. 
The paper investigates both approaches and identifies the approach of
\citeA{GrooteTv2003} as better suited for the operations of the
value iteration algorithm.
Therefore we adapt and extend  their approach 
to handle RMDPs. In particular, our
First Order Decision Diagrams
(FODD)
are defined by modifying
first order BDDs
to capture existential quantification as well as real-valued functions
through the use of an aggregation over different valuations for a diagram.
This allows us to capture MDP value functions using algebraic diagrams in a
natural way. 
We also provide
additional reduction transformations for algebraic 
diagrams that help keep their size small,
and allow the use of background knowledge in reductions. 
We then develop appropriate representations and algorithms
showing how value iteration  can be performed using FODDs.
At the core of this algorithm we introduce a novel diagram-based 
algorithm for goal
regression where, given a diagram representing the current value
function, each node in this diagram is replaced with a 
small diagram capturing its truth value before the action.
This offers a modular
and efficient form of regression that accounts for all 
potential effects of an action simultaneously.
We show that our version of abstract value iteration 
is correct and hence it converges to optimal value function and policy.

To summarize, the contributions of the paper are as follows.
The paper identifies the multiple path semantics \cite<extending>{GrooteTv2003}
as a useful representation for RMDPs and contrasts it
with the single path semantics of  \citeA{BlockeelDR98}.
The paper develops FODDs and algorithms to manipulate them in general and in the context of RMDPs.
The paper also develops novel weak reduction 
operations for first order decision diagrams and shows
their relevance to solving relational MDPs.
Finally the paper presents a version of 
the relational value iteration algorithm  using FODDs and shows
that it is correct and thus converges to the optimal
value function and policy. 
While relational value iteration was developed and specified in previous work
\cite{BoutilierRePr2001},
to our knowledge this is the first detailed proof of
correctness and convergence for the algorithm.

This section has briefly summarized the research background,
motivation, and our approach.
The rest of the paper is organized as follows.
Section 2 provides background on MDPs and RMDPs. 
Section 3 introduces the syntax and the semantics of First Order
Decision Diagrams (FODD), and 
Section 4 develops reduction operators for FODDs.
Sections 5 and 6 present a representation of RMDPs using FODDs,
the relational value iteration algorithm, and its proof of correctness
and convergence.
The last two sections conclude the paper with a discussion of the
results and future work.    

\section{Relational Markov Decision Processes}
\label{Sec:introRMDP}
We assume familiarity with standard notions of MDPs and value
iteration 
\cite<see for example>{Bellman1957,Puterman1994}.
In the following we introduce
some of the notions. We also introduce 
relational MDPs and discuss some of the previous work on solving them.

Markov Decision Processes (MDPs) provide a mathematical model of sequential
optimization problems with stochastic actions. A MDP can be
characterized by a state space $S$, an action space $A$, a state
transition function $Pr(s_j | s_i, a)$ denoting the probability of
transition to state $s_j$ given state $s_i$ and action $a$, and an
immediate reward function $r(s)$, specifying the immediate utility
of being in state $s$.
A solution to a MDP is
an optimal policy that maximizes expected discounted total reward as
defined by the Bellman equation: 
\[V^*(s) = max_{a \in A} [r(s) + {\gamma \sum_{s' \in S} Pr(s' | s, a)V^*(s')}]\]
where $V^*$ represents the optimal state-value function. 
The value iteration algorithm (VI) 
uses the Bellman
equation to iteratively refine an estimate of the value function:
\begin{equation}
\label{Eq:VI}
V_{n+1}(s) = max_{a \in A} [r(s) + \gamma \sum_{s' \in S} Pr(s' | s, a)V_{n}(s')]
\end{equation}
where $V_n(s)$ represents our current estimate of the value function
and $V_{n+1}(s)$ is the next estimate. 
If we initialize this process with $V_0$ as the reward function,
$V_n$ captures the optimal value function when we have $n$ steps to
go.
As discussed further below the algorithm is known to converge to the
optimal value function.

\citeA{BoutilierRePr2001} used the situation calculus to formalize
first order MDPs and a structured form of the value iteration algorithm.
One of the useful restrictions introduced in their work is that 
stochastic actions are specified 
as a  randomized choice among deterministic
alternatives. 
For example, action $unload$ in the logistics example
 can succeed or fail. Therefore there are two alternatives
for this action: $unloadS$ (unload success) and $unloadF$ (unload failure). 
The formulation and algorithms support any number of action alternatives.
The randomness in the domain is captured by a random choice specifying which
 action alternative ($unloadS$ or $unloadF$) gets executed 
when the agent attempts an action ($unload$).
The choice is determined by a state-dependent probability distribution 
characterizing the dynamics of the world.
In this way one can separate the regression over effects of action alternatives, which
is now deterministic, from the probabilistic choice of action. 
This considerably simplifies the reasoning required since there is no
need to perform probabilistic goal regression directly.
Most of the work on RMDPs has used this assumption, and we use this
assumption as well. \citeA{SannerBo2007} investigate a model going
beyond this assumption.

Thus relational MDPs are specified by the set of predicates in the
domain, the set of probabilistic actions in the domain, and the reward
function. For each probabilistic action, we specify the deterministic
action alternatives and their effects, and the probabilistic choice
among these alternatives. A relational MDP captures a family of MDPs
that is generated by choosing an instantiation of the state space.
Thus the logistics example corresponds to all possible instantiations
with 2 boxes or with 3 boxes and so on. We only get a concrete MDP by
choosing such an instantiation.\footnote{
One could define a single MDP including all possible instances at the
same time, e.g. it will include some states with
2 boxes, some states with 3 boxes and some with an infinite number of
boxes. But obviously subsets of these states form separate MDPs
that are disjoint. We thus prefer the view of a RMDP
as a family of MDPs.
} 
Yet our algorithms will attempt to
solve the entire MDP family simultaneously.

\citeA{BoutilierRePr2001} introduce the case notation 
to represent probabilities and rewards compactly.   
The expression
$t=case [\phi_1, t_1;\cdots; \phi_n, t_n]$, where $\phi_i$ is a logical formula, 
is equivalent to $(\phi_1\wedge (t= t_1))\vee \cdots \vee (\phi_n \wedge (t= t_n))$.
In other words, $t$ equals $t_i$ when $\phi_i$ is true. 
In general, the $\phi_i$'s are not constrained but some steps in the
VI algorithm require that
the $\phi_i$'s are disjoint and partition the state space.   
In this case, exactly one $\phi_i$ is true in any state.
Each $\phi_i$ denotes an abstract state 
whose member states have the same value for that probability or reward. 
For example, the reward function for the logistics domain, discussed
above and illustrated on the right side of Figure~\ref{Fig:Regressunload},
can be captured as $case[\exists b, Bin(b,Paris),10; \neg
  \exists b, Bin(b,Paris),0]$. 
We also have the following notation for operations over function
defined by case expressions.
The operators $\oplus$ and $\otimes$ are defined by taking a cross product of the partitions and
adding or multiplying the case values.
\[case[\phi_i, t_i : i \leq n]\oplus case[\psi_j, v_j : j \leq m]
=case[\phi_i \wedge \psi_j, t_i+ v_j: i \leq n, j \leq m]\]
\[case[\phi_i, t_i : i \leq n]\otimes case[\psi_j, v_j : j \leq m]
=case[\phi_i \wedge \psi_j, t_i \cdot v_j: i \leq n, j \leq m].\]

In each iteration of the VI algorithm,
the value of a
stochastic action $A(\vec{x})$ parameterized with free variables
$\vec{x}$ is determined in the following manner:
\begin{equation}
\label{Eq:CaseQfunction}
Q^{A(\vec{x})}(s) = rCase(s) \oplus 
[\gamma \otimes \oplus_j (pCase(n_j(\vec{x}), s) \otimes Regr(n_j(\vec{x}),vCase(do(n_j(\vec{x}), s))))]
\end{equation}
\noindent where $rCase(s)$ and $vCase(s)$ denote reward and value functions in
case notation,
$n_j(\vec{x})$ denotes the possible outcomes of the
action $A(\vec{x})$, and $pCase(n_j(\vec{x}),s)$ the choice probabilities
for $n_j(\vec{x})$. 
Note that we can replace a sum over possible next states $s'$ 
in the standard value iteration  
(Equation~\ref{Eq:VI})
with a finite sum over the action
alternatives $j$ (reflected in $\oplus_j$ in 
Equation~\ref{Eq:CaseQfunction}),
since different next states arise only through different action alternatives.

$Regr$, capturing goal regression, determines what states one must be in
before an action in order to reach a particular state after the action.  
Figure~\ref{Fig:Regressunload} illustrates the regression of 
$\exists b, Bin(b,Paris)$ in the reward function $R$ through 
the action alternative $unloadS(b^*,t^*)$. 
$\exists b, Bin(b,Paris)$ will be true after the action $unloadS(b^*,t^*)$
if it was true before or box $b^*$ was on truck $t^*$ and truck $t^*$ was in Paris.
Notice how the reward function $R$ partitions the state space into two regions or 
abstract states, each of which may include an infinite number of complete world states
(e.g., when we have an infinite number of domain objects). Also
notice how we get another set of abstract states after the regression step.
In this way first order regression 
ensures that we can work on abstract states and never need to
propositionalize the domain. 

\begin{figure}[tbhp]
\begin{center}
\setlength{\epsfxsize}{3.25in}
\centerline{\psfig{figure=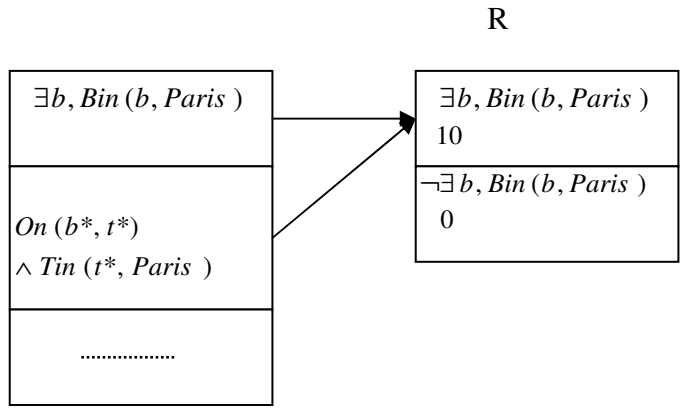}}
\caption{An example illustrating regression over the action alternative $unloadS(b^*,t^*)$.}
\label{Fig:Regressunload}
\end{center}
\vskip -0.2in
\end{figure}

After the regression, we get a parameterized $Q$-function
which accounts for all possible instances of the action.
We need to
maximize over the action parameters of the $Q$-function to
get the maximum value that could be achieved by using an instance of
this action. 
To illustrate this step,
consider the logistics example where we 
have two boxes $b_1$ and $b_2$, and $b_1$ is on truck $t_1$, which 
is in Paris (that is, $On(b_1, t_1)$ and $Tin(t_1, Paris)$), 
while $b_2$ is in Boston ($Bin(b_2, Boston)$). For the action schema $unload(b^*, t^*)$,
we can instantiate $b^*$ and $t^*$ with $b_1$ and $t_1$ respectively,
which will help us achieve the goal; or we can 
instantiate $b^*$ and $t^*$ with $b_2$ and $t_1$ respectively,
which will have no effect. Therefore we need to perform 
maximization over action parameters to get the best instance of an action.
Yet, we must perform this maximization generically, without knowledge
of the actual state.  
In SDP, this is done in several steps. First, we add 
existential quantifiers over action parameters (which leads to non
disjoint partitions).
Then we
sort the abstract states in $Q^{A(\vec{x})}$ by
the value in decreasing order and include the negated conditions for
the first $n$ abstract states in the formula for the $(n+1)^{th}$,
ensuring mutual exclusion. 
Notice how this step leads
to complex description of the resulting state partitions in SDP.
This process is performed for every action separately.
We call this step {\em object maximization}
and denote it with
$\mbox{obj-max} (Q^{A(\vec{x})})$.

Finally, to get the next value function we maximize over the 
$Q$-functions of different actions. 
These three steps provide one iteration of the VI algorithm
which repeats the update until convergence.

The solutions of ReBel \cite{KerstingOtRa2004} and 
FOVIA \cite{GroBmannHoSk2002,HolldoblerKaSk2006}
follow the same outline but use a simpler logical language
for representing RMDPs. 
An abstract state in ReBel is captured using an existentially
quantified conjunction.  
FOVIA \cite{GroBmannHoSk2002,HolldoblerKaSk2006} has a more complex
representation allowing a conjunction that must hold in a state and a
set of conjunctions that must be violated.
An important feature in ReBel is the use of
decision list \cite{Rivest1987} style
representations for value functions and policies. The decision list
gives us an implicit maximization operator since rules higher on the
list are evaluated first. As a result the object maximization step is
very simple in ReBel. Each state partition is represented implicitly by the
negation of all rules above it, and explicitly by the conjunction in the
rule. On the other hand,
regression in ReBel requires that one enumerate all possible matches 
between a subset of a conjunctive goal (or state partition) 
and action effects, and reason about each of these separately. 
So this step can potentially be improved.

In the following section we introduce a new representation --
First Order Decision Diagrams (FODD). FODDs allow for 
sharing of parts of partitions, leading to space and time saving.
More importantly the  value iteration algorithm based on FODDs has both simple
regression and simple object maximization. 

\section{First Order Decision Diagrams}

A decision diagram is a graphical representation for functions over 
propositional (Boolean) variables. The function is represented as 
a labeled rooted directed acyclic graph where each
non-leaf node is labeled with a propositional
variable and
has exactly two children.
 The outgoing edges are marked with values \true and
\falseE. Leaves are labeled with numerical values. 
Given an assignment of truth values to the propositional variables, we
can traverse the graph where in each node we follow the outgoing edge
corresponding to its truth value. This gives a mapping from any
assignment to a leaf of the diagram and in turn to its value.
If the leaves are
marked with values in $\{0,1\}$ then we can interpret the graph
as representing a Boolean function over the propositional
variables. Equivalently, the graph can be seen as representing a
logical expression which is satisfied if and only if the 1 leaf is
reached. The case with $\{0,1\}$ leaves is known as Binary Decision
Diagrams (BDDs) and the case with numerical leaves 
(or more general algebraic expressions)
is known as
Algebraic Decision Diagrams (ADDs). Decision Diagrams are particularly interesting
if we impose an order over propositional variables and require that
node labels respect this order on 
every path in the diagram; this case is known as Ordered Decision Diagrams (ODD). 
In this case every function has a unique canonical
representation that serves as a normal form for the function. 
This property   means that propositional theorem proving
is easy for ODD representations. For example, if a formula
is contradictory then this fact is evident when we represent it as a BDD, since the normal form for a 
contradiction is a single leaf valued $0$.
This property together with efficient manipulation algorithms
for ODD representations have led
 to successful applications, e.g., in
VLSI design and
verification~\cite{Bryant1992,McMillan1993,BaharFrGaHaMaPaSo1993} as
well as MDPs \cite{HoeyStHuBo1999,St-AubinHoBo2000}.
In the following we generalize this representation 
for relational problems.

\subsection{Syntax of First Order Decision Diagrams}
There are various ways to generalize ADDs to
capture relational structure. One could use closed or open formulas in the
nodes, and in the latter case we must interpret the quantification
over the variables. 
In the process of developing the ideas in this paper we have considered several possibilities
including explicit quantifiers but these did not lead to useful solutions.
We therefore focus on the following syntactic definition
which does not have any explicit quantifiers. 

For this representation, we
assume  a fixed set of predicates and constant
symbols, and an enumerable set of variables. We also allow using an
equality between any pair of terms (constants or variables).

\begin{definition}
First Order Decision Diagram
\begin{enumerate}
\item
A First Order Decision Diagram (FODD)
is a labeled rooted directed acyclic graph, where each
non-leaf node has exactly two children. 
The outgoing edges are marked with values \true and
\falseE. 
\item
Each non-leaf node is labeled with: an atom $P(t_1,\ldots,t_n)$ or an
equality $t_1=t_2$
where each $t_i$ is a variable or a constant.
\item
Leaves are labeled with numerical values. 
\end{enumerate}
\end{definition}

\begin{figure}[tbhp]
\begin{center}
\setlength{\epsfxsize}{3.25in}
\centerline{\psfig{figure=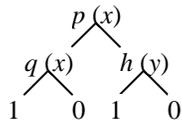}}
\caption{A simple FODD.}
\label{Fig:illsutrate-semantics}
\end{center}
\vskip -0.2in
\end{figure}

Figure~\ref{Fig:illsutrate-semantics}
shows a FODD with binary leaves. 
Left going edges represent \true branches.
To simplify 
diagrams in the paper we draw multiple copies of the leaves 0
and 1 (and occasionally other values or small sub-diagrams) 
but they represent the same node in the FODD.

We use the following notation:
for a node $n$,
$n\downt$ denotes the \true branch of $n$, and
$n\downf$ the \false branch of $n$; 
$n\downa$ is an outgoing edge from $n$, where $a$ can be \true or \falseE.
For an edge $e$, 
$source(e)$ is the node that edge $e$ issues from, and
$target(e)$ is the node that edge $e$ points to.
Let $e_1$ and $e_2$ be two edges, we have 
$e_1=sibling(e_2)$ iff $source(e_1)=source(e_2)$.

In the following we will slightly abuse the notation
and let $n\downa$ mean either an edge or the sub-FODD this edge points to.
We will also 
use $n\downa$ and $target(e_1)$
interchangeably where $n=source(e_1)$ and $a$ can be \true or \false depending
on whether $e_1$ lies in the \true or \false branch of $n$. 

\subsection{Semantics of First Order Decision Diagrams}

We use a FODD to represent a function that assigns values to states in a relational MDP.
For example, in the logistics domain, we might want to assign values to different states in such a way
that if there is a box in Paris, then the state is assigned a value
of 19; if there is no box in Paris but there is a box on a truck that is
 in Paris and it is raining, this state is assigned a value of 6.3,
 and so on.\footnote{
This is a result of regression in the logistics domain cf.\  Figure~\ref{Fig:regression}(l).}
The question is how to define the 
semantics of FODDs in order to have the intended meaning.

The semantics of first order formulas are given relative to 
interpretations. An interpretation has a domain of elements, a mapping
of constants to domain elements and, for each predicate, a
relation over the domain elements which specifies when the predicate
is true. 
In the MDP context,
a state can be captured by an interpretation.  
For example in the logistics domain, a state
includes objects such as boxes, trucks, and cities, and relations among them,
such as box 1 on truck 1 ($On(b_1, t_1)$), box 2 in Paris ($Bin(b_2,
Paris)$) and so on.
There is more than one way to define the meaning of FODD $B$ on
interpretation $I$. 
In the following we discuss two possibilities.

\subsubsection{Semantics Based on a Single Path}

A semantics for relational decision trees is given by 
\citeA{BlockeelDR98} and it can be adapted to FODDs.
The semantics define a unique path
that is followed when traversing $B$ relative to $I$. 
All variables are existential and a node is evaluated relative to the
path leading to it.

In particular,
when we reach a node some of its variables have been seen before on
the path and some are new.  Consider a node $n$ with label $l(n)$ and the
path leading to it from the root, and let $C$ be the conjunction of
all labels of nodes that are exited on the \true branch on the
path. Then in the node $n$ we evaluate $\exists \vec{x}, C\wedge
l(n)$, where $\vec{x}$ includes all the variables in $C$ and $l(n)$. 
If this
formula is satisfied in $I$ then we follow the \true
branch. Otherwise we follow the \false branch. This process defines
a unique path from the root to a leaf and its value. 

For example, if we evaluate the diagram in
Figure~\ref{Fig:illsutrate-semantics} on the interpretation $I_1$ with
domain $\{1,2,3\}$ and where the only true atoms are $\{p(1),q(2),h(3)\}$ then we follow
the \true branch at the root since $\exists x, p(x)$ is satisfied,
but we follow the \false branch 
at $q(x)$
since $\exists x, p(x)\wedge q(x)$ is
not satisfied. Since the leaf is labeled with 0 we say that $B$ does
not satisfy $I$. 
This is an attractive approach, 
because it partitions the set of interpretations into 
mutually exclusive sets and this can be used to create abstract
state partitions in the MDP context. 
However, for reasons we discuss later, this semantics 
leads to various complications for the value iteration algorithm,
and it is therefore not used in
the paper.

\subsubsection{Semantics Based on  Multiple Paths}

The second alternative builds on work by \citeA{GrooteTv2003} who defined
semantics based on multiple paths.
Following this work, we define the semantics first relative
to a variable valuation $\zeta$.
Given a FODD $B$ over variables $\vec{x}$
and an interpretation $I$, a valuation $\zeta$ maps each variable in
$\vec{x}$ to a domain element in $I$. Once this is done, each node
predicate evaluates either to \true or \false and 
we can traverse a single path to a leaf. The value of this leaf is
denoted by $\map_B(I,\zeta)$. 

Different valuations may give different values; but recall that 
we use FODDs to represent a function over states, and each state must be
assigned a single value. Therefore,
we next define 
\[\map_B(I) =
\mbox{aggregate}_{\zeta}\{\map_B(I,\zeta)\}\] 
for some aggregation
function. That is, we consider all possible valuations $\zeta$, and for
each valuation we calculate $\map_B(I,\zeta)$. We then  aggregate over all
these values. In the special case of~\citeA{GrooteTv2003} 
leaf labels are in
$\{0,1\}$
and variables are 
universally quantified;
this is easily captured in our formulation  by using minimum as the aggregation function.
In this paper we use maximum as the aggregation function. This
corresponds to existential quantification in the binary case
(if there is a valuation leading to value $1$, then the value assigned will be $1$)
 and gives useful maximization for value functions in the general case. 
We therefore define:
\[\map_B(I) =
\max_{\zeta}\{\map_B(I,\zeta)\}. \]
Using this definition $B$ assigns every $I$ a unique value
$v=\map_B(I)$
so $B$ defines a function from interpretations to real values. 
We later refer to this function as {\em the map of $B$}.

Consider evaluating the diagram in
Figure~\ref{Fig:illsutrate-semantics} on the interpretation $I_1$
given above where the only true atoms are
$\{p(1),q(2),h(3)\}$. 
The valuation where $x$ is mapped to $2$ and $y$
is mapped to 3 denoted $\{x/2, y/3\}$ leads to a leaf with value 1
so the maximum is 1. When leaf labels are in \{0,1\}, we can interpret 
the diagram as a logical formula.
When $\map_B(I) = 1$, as in our example, we say that $I$ satisfies $B$ and 
when $\map_B(I) = 0$ we say that $I$ falsifies $B$. 

We define node formulas (NF) and edge formulas (EF) recursively as
follows.  For a node $n$ labeled $l(n)$ with incoming edges $e_1, \ldots,
e_k$, the node formula $\NF(n) = (\vee_i \EF(e_i))$.  
The edge
formula for the \true outgoing edge of $n$ is $\EF(n\downt) =
\NF(n)\wedge l(n)$.  The edge formula for the \false outgoing edge
of $n$ is $\EF(n\downf) = \NF(n)\wedge \neg l(n)$.  These
formulas, where all variables are existentially quantified, capture the
conditions under which a node or edge are reached.

\subsection{Basic Reduction of FODDs}
\label{Sec:basicreduction}
\citeA{GrooteTv2003} define several operators that reduce a
diagram into normal form. 
A total order over node labels is assumed. 
We describe these operators briefly and
give their main properties. 

\begin{description}
\item[(R1)]
Neglect operator: if both children of a node $p$ in the FODD lead to
    the same node $q$ then we remove $p$ and link all parents of
    $p$ to $q$ directly.
\item[(R2)]
 Join operator: if two nodes $p, q$ have the same label and point to
    the same two children then we can join $p$ and $q$ (remove $q$ and
    link $q$'s parents to $p$).
\item[(R3)]
 Merge operator: if a node and its child have the same label then
    the parent can point directly to the grandchild.
\item[(R4)]
 Sort operator: If a node $p$ is a parent of $q$ but the label ordering is
    violated ($l(p)>l(q)$)
then we can reorder the nodes locally using two copies of $p$ and $q$
such that labels of the nodes do not violate the ordering. 
\end{description}

Define a FODD to be reduced if none of the four operators can be
applied. We have the following:

\begin{theorem} {\bf \cite{GrooteTv2003}}\ \\
(1) Let $O\in\{$Neglect, Join, Merge, Sort$\}$ be an operator and $O(B)$
the result of applying $O$ to FODD $B$, then for any $B$, $I$, and $\zeta$,
$\map_B(I,\zeta) = \map_{O(B)}(I,\zeta)$.
\\
(2) If $B_1, B_2$ are reduced and satisfy
$\forall \zeta,$ $\map_{B_1}(I,\zeta)$ $= \map_{B_2}(I,\zeta)$
then they are identical.
\end{theorem}
\vspace{-2mm}
Property (1) gives soundness, and property (2) shows that reducing a 
FODD gives a normal form.
However, this only holds if the maps are
identical for every $\zeta$ and this condition is stronger than normal
equivalence. 
This  normal form suffices for \citeA{GrooteTv2003} 
who use it to provide a theorem prover for first order logic,
but it is not strong enough for our purposes. 
Figure~\ref{Fig:reduced} shows two pairs of reduced
FODDs (with respect to R1-R4)
 such that $\map_{B_1}(I) = \map_{B_2}(I)$ but 
$\exists \zeta, \map_{B_1}(I,\zeta) \not = \map_{B_2}(I,\zeta)$.
In this case although the maps are the same the FODDs are not reduced to
the same form.
Consider first the pair in part (a) of the figure.
An interpretation where $p(a)$ is false but $p(b)$ is true and a
substitution $\{x/a, y/b\}$ leads to value of 0 in $B_1$
while $B_2$ always evaluates to 1.
But the diagrams are equivalent. For any interpretation, 
if $p(c)$
is true for any object $c$ then $\map_{B_1}(I)=1$ through the
substitution $\{x/c\}$;
if $p(c)$
is false for any object $c$ then $\map_{B_1}(I)=1$ through the
substitution $\{x/c, y/c\}$. Thus the map is always 1 for $B_1$ as well.
In Section~\ref{Sec:R7} we  show that with the additional
reduction operators we have developed, B1 in the first pair
is reduced to $1$. Thus the diagrams in (a) have the same form after 
reduction. However,
our reductions do not resolve the second pair given in part (b) of the
figure.
Notice that both  functions capture 
a path of two edges labeled $p$ in a graph
(we just change the order of two nodes and rename variables)
so the diagrams evaluate to 1 if and only if the interpretation has
such a path.
Even though B1 and B2 are logically equivalent,
they cannot be reduced to the same form using R1-R4 or our new
operators. To identify
a unique minimal syntactic form one may have to consider all possible renamings
of variables and the sorted diagrams they produce, but this is an expensive operation.
A discussion of normal form for conjunctions that uses 
such an operation is given by \citeA{GarrigaKhRa2007}.

\begin{figure}[tb]
\begin{center}
\centerline{\psfig{figure=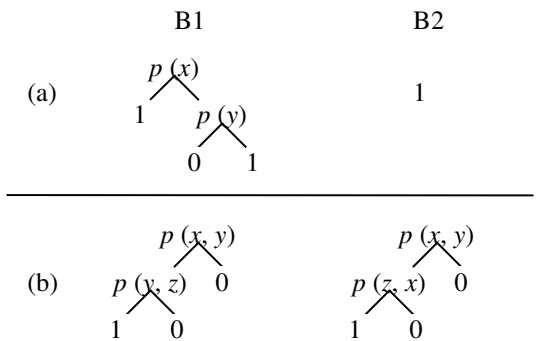}}
\caption{Examples illustrating weakness of normal form.}
\label{Fig:reduced}
\end{center}
\vskip -0.2in
\end{figure}

\subsection{Combining FODDs} 

Given two algebraic diagrams we
may need to add the corresponding functions, take the maximum or use
any other binary operation, {\tt op}, over the values represented by the
functions.  
Here we adopt the solution from the propositional case
\cite{Bryant1986}
in the form of the procedure
{\bf Apply($B_1$,$B_2$,{\tt op})} where $B_1$ and $B_2$ are 
algebraic diagrams.
Let $p$ and $q$ be the roots of $B_1$ and $B_2$ respectively.
This procedure chooses a new root label (the lower among
labels of $p,q$) and recursively combines the corresponding
sub-diagrams, according to the relation between the two labels ($\prec$,
$=$, or $\succ$).
In order to make sure the result is reduced in the propositional
sense one can use dynamic programming to avoid generating nodes for
which either neglect or join operators ((R1) and (R2) above) would be
applicable. 

\begin{figure}[tbhp]
\vskip 0.2in
\begin{center}
\setlength{\epsfxsize}{3.25in}
\centerline{\psfig{figure=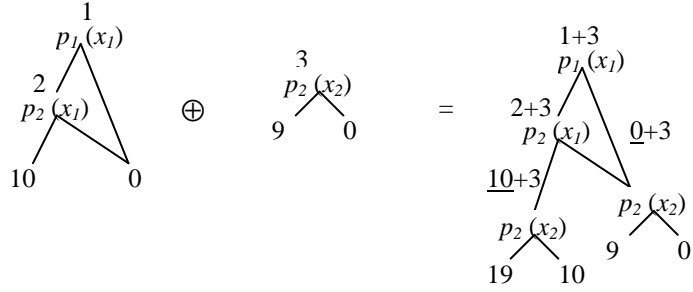}}
\caption{A simple example of adding two FODDs.}
\label{Fig:combineFODD}
\end{center}
\vskip -0.2in
\end{figure} 

Figure~\ref{Fig:combineFODD} illustrates this process. In this example, we 
assume predicate ordering as $p_1 \prec p_2$, and parameter ordering
$x_1 \prec x_2$.
Non-leaf nodes are annotated with numbers and numerical leaves are underlined 
for identification during the execution trace.
For example, the top level call adds the functions corresponding to
nodes 1 and 3. Since $p_1(x_1)$ is the smaller label it is picked as
the label for the root of the result. Then we must add both left and
right child of node 1 to node 3. These calls are performed recursively.
It is easy to see that the size of the result may be the product of sizes of input diagrams. 
However, much pruning will occur with shared variables and further pruning 
is made possible by weak reductions presented later.
 
Since for any interpretation $I$ and any fixed valuation $\zeta$
the FODD is propositional, we have the following lemma.
We later refer to this property as the {\em correctness
of Apply}. 

\begin{lemma}
Let $C=Apply(A, B, op)$, then for any $I$ and $\zeta$, 
$\map_{A}(I,\zeta) \mbox{ op } \map_{B}(I,\zeta)= \map_{C}(I,\zeta)$.
\end{lemma}

\begin{proof}
First we introduce some terminology. 
Let $\#nodes(X)$ refer to the set of all nodes in a FODD $X$. 
Let the root nodes of $A$ and $B$ be $A_{root}$ and $B_{root}$ respectively. 
Let the FODDs rooted at 
$A_{root\downt}$, 
$A_{root\downf}$, 
$B_{root\downt}$, 
$B_{root\downf}$, 
$C_{root\downt}$, and 
$C_{root\downf}$ be 
$A^l$, $A^r$, $B^l$, $B^r$, $C^l$ and $C^r$ respectively.

The proof is by induction on $n = |\#nodes(A)| + |\#nodes(B)|$. 
The lemma is true for $n = 2$, because in this case 
both $A_{root}$ and $B_{root}$ have to be single leaves and 
an operation on them is the same as an operation on two real numbers.
For the inductive step we need to consider two cases.

Case 1: $A_{root} = B_{root}$. 
Since the root nodes are equal, if a valuation $\zeta$ reaches $A^l$, 
then it will also reach $B^l$ and if $\zeta$ reaches $A^r$, 
then it will also reach $B^r$. 
Also, by the definition of Apply, in this case $C^l = Apply(A^l, B^l, op)$ 
and $C^r = Apply(A^r, B^r, op)$. 
Therefore the statement of the lemma is true if
$\map_{A^l}(I, \zeta) \mbox{ op } \map_{B^l}(I, \zeta) = \map_{C^l}(I, \zeta)$ and 
$\map_{A^r}(I, \zeta) \mbox{ op } \map_{B^r}(I, \zeta) = \map_{C^r}(I, \zeta)$ 
for any $\zeta$ and $I$. 
Now, since $|\#nodes(A^l) + \#nodes(B^l)| < n$ and $|\#nodes(A^r) + \#nodes(B^r)| < n$, 
this is guaranteed by the induction hypothesis.  

Case 2:
$A_{root} \neq B_{root}$. 
Without loss of generality let us assume that $A_{root} \prec B_{root}$. 
By the definition of Apply, 
$C^l = Apply(A^l, B, op)$ and 
$C^r = Apply(A^r, B, op)$. 
Therefore the statement of the lemma is true if 
$\map_{A^l}(I, \zeta) \mbox{ op } \map_{B}(I, \zeta) = \map_{C^l}(I, \zeta)$ and 
$\map_{A^r}(I, \zeta) \mbox{ op } \map_{B}(I, \zeta) = \map_{C^r}(I, \zeta)$ for any $\zeta$ and $I$. 
Again this is guaranteed by the induction hypothesis.
\end{proof}

\subsection{Order of Labels}
\label{Sec:labelorder}

The syntax of FODDs allows for two ``types'' of objects:
constants and variables. 
Any argument of 
a predicate can be a constant or a variable. 
We assume a complete ordering on predicates, constants, 
and variables. 
The ordering 
$\prec$ between two labels is given by the following rules.

\begin{enumerate}
\item
$P(x_1,...,x_n)\prec P'(x_1',...,x_m')$ if $P\prec P'$
\item
$P(x_1,...,x_n)\prec P(x_1',...,x_n')$ if there 
exists $i$ such that 
$x_j=x_j'$ for all $j<i$, and
$type(x_i)\prec type(x_i')$ 
(where ``type'' can be constant or variable) or
$type(x_i)=type(x_i')$ and $x_i\prec x_i'$. 
\end{enumerate}

While the predicate order can be set arbitrarily it appears useful to assign
the equality predicate
as the first in the predicate ordering so that
equalities are at the top of the diagrams.
During reductions we often encounter situations where
one side of the equality  can be completely removed
leading to substantial space savings.
It may also be useful to order the argument types 
so that
constant $\prec$ variables.
This ordering may be helpful for reductions.
Intuitively, a variable appearing lower in the diagram can be bound to 
the value of a constant that appears above it.
These are only heuristic guidelines and the best ordering may well be
problem dependent.
We later introduce other forms of arguments: 
{\em predicate parameters} and {\em action parameters}.
The ordering for these is discussed in Section~\ref{Sec:VI}.

\section{Additional Reduction Operators}
In our context, especially for algebraic FODDs, we may want to reduce
the diagrams further. We distinguish {\em strong reductions} that
preserve $\map_{B}(I,\zeta)$ for all $\zeta$ and {\em weak reductions}
that only preserve $\map_{B}(I)$. 
Theorem 1 shows that R1-R4 given above are strong reductions.
The details of our relational VI algorithm do not
directly depend on the reductions used. 
Readers more interested in RMDP details 
can skip to Section~\ref{Sec:relationalRep} 
which can be read independently 
(except where reductions are illustrated in examples).

All the reduction operators below can incorporate existing knowledge on
relationships between predicates in the domain. 
We denote this background knowledge by
$\B$.
For example in the
Blocks World we may know that if there is a block on block $y$ then it
is not clear: $\forall x,y, [on(x,y)\rightarrow \neg
clear(y)]$.

In the following when we define conditions for reduction operators,
there are two types of conditions: the reachability condition and the value condition. 
We name reachability conditions by starting with P (for Path Condition)
and the reduction operator number. We name conditions on values by starting
with V and the reduction operator number.

\subsection{(R5) Strong Reduction for Implied Branches}
Consider any node $n$ such that  whenever $n$ is reached then the \true
branch is followed. In this case we can remove
$n$ and connect its parents  directly to the \true branch.
We first present the condition, followed by the lemma 
regarding this operator.

{\cond {P5}:}
$\B\models \forall \vec{x}, [\NF(n)\rightarrow l(n)]$
where  $\vec{x}$ are the variables in $\EF(n\downt)$.

Let $\rfiveremove (n)$ denote the operator that removes
node $n$ and connects its parents directly to the \true branch.
Notice that this is a generalization of R3.
It is easy to see that the following lemma is true:

\begin{lemma}
Let $B$ be a FODD, $n$ a node for which condition P5 holds, and 
$B'$ the result of $\rfiveremove(n)$. Then for any interpretation
$I$ and any valuation $\zeta$ we have $\map_B(I,\zeta)=\map_{B'}(I,\zeta)$.
\end{lemma}

A similar reduction can be formulated for the \false branch,
i.e., if $\B\models \forall \vec{x}, [\NF(n)\rightarrow \neg l(n)]$
then whenever node $n$ is reached then the \false
branch is followed. 
In this case we can remove
$n$ and connect its parents  directly to the \false branch.

Implied branches may simply be a
result of equalities along a path. For example $(x=y) \wedge
p(x)\rightarrow p(y)$ so we may prune $p(y)$ if $(x=y)$ and 
$p(x)$ are known to be true.
Implied branches may also be a result of background knowledge.
For example in the Blocks World 
if $on(x,y)$ is guaranteed to be true when we reach a node
labeled $clear(y)$ then we can remove 
$clear(y)$ and connect its parent to $clear(y)\downf$.

\subsection{(R7) Weak Reduction Removing Dominated Edges}
\label{Sec:R7}

Consider any two edges $e_1$ and $e_2$ in a FODD
whose formulas satisfy that if we can follow $e_2$
using some valuation
then we can also follow $e_1$
using a possibly different valuation.
If $e_1$ gives better value than $e_2$
then intuitively $e_2$ never determines the value of the diagram and
is therefore redundant.
We formalize this as reduction operator
R7.\footnote{We 
use R7 and skip the notation R6 for consistency 
with earlier versions of this paper. 
See further discussion in Section~\ref{Sec:R6}.}

Let $p=source(e_1), q=source(e_2)$, 
$e_1=p\downa$, and
$e_2=q\downb$, where $a$ and $b$
can be \true or \falseE.
We first present all the conditions for the operator
and then follow with the definition of the operator.

{\cond {P7.1}:}
$\B\models [\exists \vec{x}, \EF(e_2)]
\rightarrow [\exists \vec{y}, \EF(e_1)]$
where $\vec{x}$ are the variables in $\EF(e_2)$ 
and $\vec{y}$ the variables in $\EF(e_1)$.

{\cond {P7.2}:}
$\B\models \forall \vec{u}, [
             [\exists \vec{w}, \EF(e_2)]\rightarrow 
             [\exists \vec{v}, \EF(e_1)]]$ 
where  $\vec{u}$ are the variables that appear in both $target(e_1)$ and $target(e_2)$,
$\vec{v}$ the variables that appear in $\EF(e_1)$ 
but are not in $\vec{u}$,
and
$\vec{w}$ the variables that appear in $\EF(e_2)$ 
but are not in $\vec{u}$.
This condition requires that for every valuation $\zeta_1$ that reaches $e_2$ there is a
valuation $\zeta_2$ that reaches $e_1$ such that 
$\zeta_1$ and $\zeta_2$ agree on all  variables 
that appear in both 
$target(e_1)$ and $target(e_2)$. 

{\cond {P7.3}:} 
$\B\models \forall \vec{r}, [
             [\exists \vec{s}, \EF(e_2)]\rightarrow 
             [\exists \vec{t}, \EF(e_1)]]$
where $\vec{r}$ are the variables that appear in both $target(e_1)$ and $target(sibling(e_2))$,
$\vec{t}$ the variables that appear in $\EF(e_1)$ 
but are not in $\vec{r}$,
and
$\vec{s}$ the variables that appear in $\EF(e_2)$ 
but are not in $\vec{r}$.
This condition requires that for every valuation
$\zeta_1$ that reaches $e_2$ there is a valuation $\zeta_2$ that reaches $e_1$
such that $\zeta_1$ and $\zeta_2$ agree on all variables 
that appear in both
$target(e_1)$ and $target(sibling(e_2))$. 

{\cond {V7.1}:} $\min(target(e_1))\geq\max(target(e_2))$
where $\min(target(e_1))$ is the minimum leaf value in $target(e_1)$,
and $\max(target(e_2))$ the maximum leaf value in $target(e_2)$.
In this case
regardless of the valuation we know that it is better to follow
$e_1$ and not $e_2$.
 
{\cond {V7.2}:} $\min(target(e_1))\geq\max(target(sibling(e_2)))$.

{\cond {V7.3}:} all leaves in $D=target(e_1)\ominus target(e_2)$ have non-negative values, 
denoted as $D \geq 0$.  
In this case
for any fixed valuation
it is better to follow $e_1$  instead of $e_2$. 

{\cond {V7.4}:} all leaves in $G=target(e_1) \ominus target(sibling(e_2)) $ have non-negative values.

We define the operators $\rsevenreplace(b,e_1,e_2)$ as replacing
$target(e_2)$  with a constant $b$ that is between 0 and $\min(target(e_1))$
(we may write it as $\rsevenreplace(e_1,e_2)$ if $b=0$), and
$\rsevendrop(e_1,e_2)$ as dropping the node $q=source(e_2)$ and connecting
its parents to $target(sibling(e_2))$.

We need one more ``safety'' condition to guarantee that the reduction is correct:

{\cond {S1}:} $\NF(source(e_1))$ and the sub-FODD of $target(e_1)$ remain the same
before and after R7-replace and R7-drop. 
This condition says that we must not harm the value promised by $target(e_1)$. In other words, 
we must guarantee that $p=source(e_1)$ is reachable just as before and the sub-FODD of
$target(e_1)$ is not modified after replacing a branch with $0$. 
The condition is violated
if $q$ is in the sub-FODD of $p\downa$, or
if $p$ is in the sub-FODD of $q\downb$.
But it holds in all other cases, that is when
$p$ and $q$ are unrelated (one is not the descendant of the other), or
$q$ is in the sub-FODD of $p\downna$, or
$p$ is in the sub-FODD of $q\downnb$, where $\overline{a},
\overline{b}$ are the negations of $a, b$.

\begin{lemma}
\label{Lemma:R7replace1}
Let $B$ be a FODD, $e_1$ and $e_2$ edges for which conditions P7.1, V7.1, and S1 hold, and 
$B'$ the result of $\rsevenreplace(b,e_1,e_2)$, 
where  $0 \leq b \leq \min(target(e_1))$,
then for any interpretation $I$
we have $\map_B(I)=\map_{B'}(I)$. 
\end{lemma}

\begin{proof}
Consider any valuation $\zeta_1$ that reaches $target(e_2)$.
Then according to P7.1, there is another valuation reaching $target(e_1)$ 
 and by V7.1 it
gives a higher value. Therefore,
$\map_B(I)$ will never
be determined by $target(e_2)$ so we can 
replace $target(e_2)$  with a constant between 0 and $\min(target(e_1))$
without changing the map.
\end{proof}

\begin{lemma}
\label{Lemma:R7replace2}
Let $B$ be a FODD, $e_1$ and $e_2$ edges for which conditions P7.2, V7.3, and S1 hold, and 
$B'$ the result of $\rsevenreplace(b,e_1,e_2)$, 
where  $0 \leq b \leq \min(target(e_1))$,
then for any interpretation $I$
we have $\map_B(I)=\map_{B'}(I)$. 
\end{lemma}

\begin{proof}
Consider any valuation $\zeta_1$
that reaches $target(e_2)$. By P7.2 there is another valuation $\zeta_2$ reaching $target(e_1)$
and $\zeta_1$ and $\zeta_2$ agree on all variables that appear in 
both $target(e_1)$ and $target(e_2)$. Therefore,
by V7.3 it achieves a higher value 
(otherwise, there must be a branch in $D=target(e_1) \ominus target(e_2)$ 
with a negative value). Therefore according to maximum aggregation the value
of $\map_B(I)$ will never
be determined by $target(e_2)$, and we
can replace it with a constant as described above. 
\end{proof}

Note that the conditions in the previous two lemmas are not comparable
since P7.2 $\rightarrow$ P7.1 and V7.1 $\rightarrow$ V7.3. 
Intuitively when we relax the conditions on values, 
we need to strengthen the conditions on reachability.
The subtraction operation $D=target(e_1) \ominus target(e_2)$ is
propositional,
so the test in V7.3 implicitly assumes  
that the common variables in the operands are
the same and P7.1 does not check this.
Figure~\ref{Fig:R7replacew0}
illustrates that
the reachability condition P7.1 
together with V7.3,
i.e., combining the weaker portions of conditions 
from Lemma~\ref{Lemma:R7replace1}  and Lemma~\ref{Lemma:R7replace2},
 cannot guarantee
that we can replace a branch with a constant. 
 Consider an interpretation $I$ 
with domain $\{1,2,3,4\}$ and relations 
$\{h(1,2),q(3,4),p(2)\}$. In addition assume domain knowledge 
 $\B=[\exists x,y, h(x,y) \rightarrow \exists z,w,q(z, w)]$.
So P7.1 and V7.3 hold for $e_1=[q(x,y)]\downt$ and $e_2=[h(z,y)\downt]$.
We have
$\map_{B1}(I)=3$
and 
$\map_{B2}(I)=0$. 
It is therefore not possible to replace
$h(z,y)\downt$ with $0$.

\begin{figure}[tbhp]
\vskip 0.2in
\begin{center}
\centerline{\psfig{figure=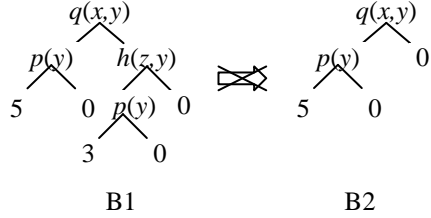}}
\caption{An example illustrating the subtraction condition in R7.}
\label{Fig:R7replacew0}
\end{center}
\vskip -0.2in
\end{figure}

\begin{figure}[tbhp]
\vskip 0.2in
\begin{center}
\centerline{\psfig{figure=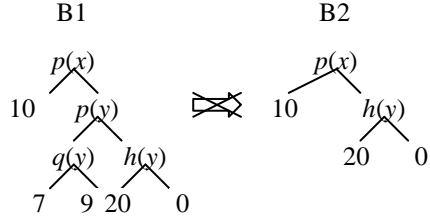}}
\caption{An example illustrating the condition for removing a node in R7.}
\label{Fig:R7remove}
\end{center}
\vskip -0.2in
\end{figure} 

Sometimes we can drop the node $q$ completely with $\rsevendrop$.
Intuitively, when we remove a node, we must guarantee that
we do not gain extra value.
The conditions for $\rsevenreplace$ can only guarantee that we
will not lose any value.
 But if we remove 
the node $q$, a valuation that was supposed to reach $e_2$ may reach
a better value in $e_2$'s sibling.
This would change the map, as illustrated in 
Figure~\ref{Fig:R7remove}.
Notice that the conditions P7.1 and V7.1 hold for 
$e_1=[p(x)]\downt$ and
$e_2=[p(y)]\downt$ so we can replace
$[p(y)]\downt$ with a constant.
 Consider an interpretation $I$ with
domain $\{1,2\}$ and relations $\{q(1),p(2),h(2)\}$. 
We have
$\map_{B1}(I)=10$ 
via valuation  $\{x/2\}$
and 
$\map_{B2}(I)=20$
via valuation $\{x/1, y/2\}$. 
Thus removing $p(y)$ is not correct.

Therefore we need the additional condition  to guarantee that 
we will not gain extra value with node dropping.
This condition can be stated as: for any valuation $\zeta_1$ that reaches $e_2$ 
and thus will be redirected to reach a value $v_1$ in $sibling(e_2)$ when
$q$ is removed, there is a valuation $\zeta_2$
that reaches a leaf with value $v_2 \geq v_1$.  
However, this condition is too complex to test in practice.
In the following we identify two stronger conditions.

\begin{lemma}
\label{Lemma:R7drop1}
Let $B$ be a FODD, $e_1$ and $e_2$ edges for which condition V7.2 hold 
in addition to the conditions for replacing $target(e_2)$  with a constant, and 
$B'$ the result of $\rsevendrop(e_1,e_2)$, then for any interpretation $I$
we have $\map_B(I)=\map_{B'}(I)$. 
\end{lemma}

\begin{proof}
Consider any valuation reaching $target(e_2)$.
As above its true value is dominated by another valuation reaching
$target(e_1)$. When we remove $q=source(e_2)$ the valuation 
will reach $target(sibling(e_2))$ and by V7.2 the value produced is smaller
than the value from $target(e_1)$. So again the map is preserved.
\end{proof}

\begin{lemma}
\label{Lemma:R7drop2}
Let $B$ be a FODD, $e_1$ and $e_2$ edges for which P7.3 and V7.4 hold
in addition to conditions for replacing $target(e_2)$  with a constant, and 
$B'$ the result of $\rsevendrop(e_1,e_2)$, then for any interpretation $I$
we have $\map_B(I)=\map_{B'}(I)$. 
\end{lemma}

\begin{proof}
Consider any valuation $\zeta_1$ reaching $target(e_2)$. As above its value 
is dominated by another valuation reaching $target(e_1)$. When we remove $q=source(e_2)$
the valuation will reach $target(sibling(e_2))$ and by the conditions P7.3 and V7.4, 
the valuation $\zeta_2$ will reach 
leaf of greater value in $target(e_1)$(otherwise there will be a branch in G leading to a
negative value). So under maximum aggregation the map
is not changed.
\end{proof}

To summarize if 
P7.1 and V7.1 and S1 hold 
or P7.2 and V7.3 and S1 hold 
then we can replace $target(e_2)$ with a
constant. 
If we can replace and V7.2 or both P7.3 and V7.4 hold then we can drop 
$q=source(e_2)$ completely.
 
In the following we provide a more detailed analysis of applicability and variants of R7.

\subsubsection{R6: A Special Case of R7}
\label{Sec:R6}
We have a special case of R7 when $p=q$, 
i.e., $e_1$ and $e_2$ are siblings.
In this context R7 can be considered to 
focus on a single node $n$ instead of two edges.
Assuming that $e_1=n\downt$ and $e_2=n\downf$,
we can rewrite the conditions in R7 as follows.

{\cond {P7.1}:}
$\B\models [\exists \vec{x}, \NF(n)]\rightarrow [\exists
  \vec{x},\vec{y}, \EF(n\downt)]$.
This condition requires that if $n$ is reachable then $n\downt$ is reachable.

{\cond {P7.2}:}
 $\B\models \forall \vec{r}, [\exists \vec{v}, \NF(n)]\rightarrow
           [\exists \vec{v}, \vec{w}, \EF(n\downt)]$
where $\vec{r}$ are the variables that appear in both $n\downt$ and $n\downf$,
$\vec{v}$ the variables that appear in $\NF(n)$ but not in $\vec{r}$, 
and $\vec{w}$ the variables in $l(n)$ and not in $\vec{r}$ or $\vec{v}$.

{\cond {P7.3}:}
$\B\models \forall \vec{u}, [\exists \vec{v}, \NF(n)]\rightarrow
           [\exists \vec{v}, \vec{w}, \EF(n\downt)]$
where $\vec{u}$ are the variables that appear in $n\downt$ 
(since $sibling(e_2)=e_1$),
$\vec{v}$ the variables that appear in $\NF(n)$ but not in $\vec{u}$,
and $\vec{w}$ the variables in $l(n)$ and not in $\vec{u}$ or $\vec{v}$.

{\cond {V7.1}:} $\min(n\downt)\geq\max(n\downf)$.

{\cond {V7.2}:} $n\downt$ is a constant.

{\cond {V7.3}:} all leaves in the diagram $D=n\downt \ominus n\downf$ have non-negative values.

Conditions S1 and V7.4 are always true.
We have previously analyzed this special case as a separate
reduction operator named R6 \cite{WangJoKh2007}.
While this is a special case, it may still be useful to check for it
separately before applying the generalized case
of R7, as it provides large reductions and seems to occur
frequently in example domains.

An important special case of R6 occurs when $l(n)$ is an 
equality $t_1 = y$ where $y$ is a variable that does not occur in the
FODD above node $n$. In this case, the condition P7.1
holds since we can choose the value of $y$. 
We can also enforce the equality in the sub-diagram of $n\downt$.
Therefore if V7.1 holds we can remove the node $n$ connecting its parents to
$n\downt$ and substituting $t_1$ for $y$ in the diagram $n\downt$.
(Note that we may need to make copies of  nodes when doing this.)
In Section~\ref{Sec:equality} we introduce
a more elaborate reduction to handle equalities 
by taking a maximum over the left and the right children.

\subsubsection{Application Order}
  
In some cases several instances of R7 are applicable. It turns out that
the order in which we apply them is important.
In the following, the first example shows that 
the order affects the number of steps needed to reduce the diagram.
The second example shows that the order affects 
the final result. 

Consider the FODD in Figure~\ref{Fig:R7order}(a). R7 is applicable to edges
 $e_1=[p(x_1,y_1)]\downt$ and $e_2=[p(x_2,y_2)]\downt$, 
and $e_1'=[q(x_3)]\downt$ and $e_2'=[q(x_2)]\downt$. 
If we reduce in a top down manner, i.e., first apply R7 on the pair
 $[p(x_1,y_1)]\downt$ and $[p(x_2,y_2)]\downt$, we will get the FODD in Figure~\ref{Fig:R7order}(b),
and then we apply R7 again on $[q(x_3)]\downt$ and $[q(x_2)]\downt$, and we will get the FODD in 
Figure~\ref{Fig:R7order}(c). However,
if we apply R7 first on $[q(x_3)]\downt$ and $[q(x_2)]\downt$ thus getting Figure~\ref{Fig:R7order}(d),
R7 cannot be applied to
 $[p(x_1,y_1)]\downt$ and $[p(x_2,y_2)]\downt$
 because  $[p(x_1,y_1)]\downt \ominus [p(x_2,y_2)]\downt$ will have negative leaves.
In this case, the diagram can still be reduced.
We can reduce 
by comparing $[q(x_3)]\downt$ and $[q(x_2)]\downt$ that is in the right part of FODD. 
We can first remove $q(x_2)$ and get a FODD shown 
in Figure~\ref{Fig:R7order}(e),
 and then use the neglect operator to remove $p(x_2,y_2)$.
As we see in this example applying one instance of R7
may render other instances not applicable 
or may introduce more possibilities for reductions 
so in general we must apply the reductions sequentially.
\citeA{Wang2007} develops conditions under which 
several instances of R7 can be applied simultaneously.

\begin{figure}[t]
\begin{center}
\centerline{\psfig{figure=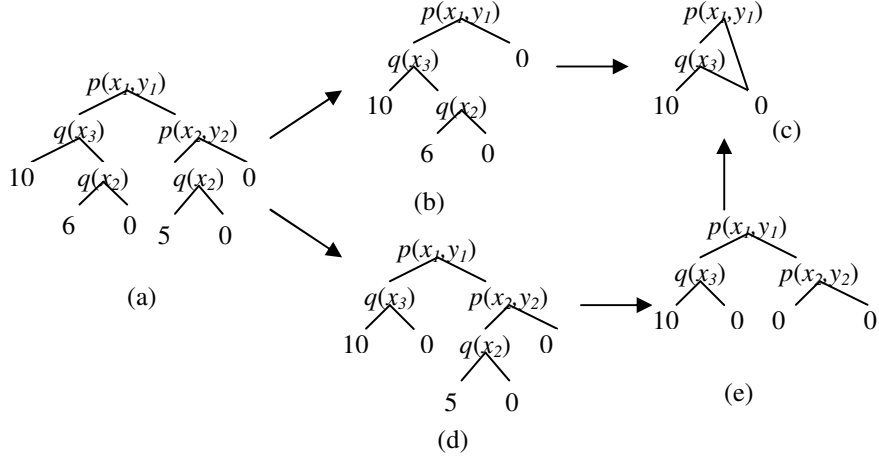}}
\caption{An example illustrating the effect of application order for R7.}
\label{Fig:R7order}
\end{center}
\vskip -0.2in
\end{figure}

One might hope that repeated application of R7 will lead to a unique reduced result but this 
is not true.
In fact, the final result depends on the choice of operators and the order of application. 
Consider Figure~\ref{Fig:R7order3}(a). 
 R7 is applicable to edges
 $e_1=[p(x)]\downt$ and $e_2=[p(y)]\downt$, 
and $e_1'=[q(x)]\downt$ and $e_2'=[q(y)]\downt$. 
If we reduce in a top down manner, i.e., first apply R7 on the pair
 $[p(x)]\downt$ and $[p(y)]\downt$, we will get the FODD in Figure~\ref{Fig:R7order3}(b),
which cannot be reduced using existing reduction operators (including 
the operator R8 introduced below).
However,
if we apply R7 first on $[q(x)]\downt$ and $[q(y)]\downt$ 
we will get Figure~\ref{Fig:R7order3}(c). Then we can apply R7 again on
$e_1=[p(x)]\downt$ and $e_2=[p(y)]\downt$ and get the final result
Figure~\ref{Fig:R7order3}(d), which is clearly more compact than Figure~\ref{Fig:R7order3}(b).
It is interesting that the first example seems to suggest 
applying R7 in a top down manner (since it takes fewer steps), 
while the second 
seems to suggest the opposite (since the final result is more compact). 
More research is needed to develop
useful heuristics to guide the choice of reductions and the application order
and in general develop a more complete set of reductions.

Note that we could also consider generalizing R7. 
In Figure~\ref{Fig:R7order3}(b), if we can reach $[q(y)]\downt$ then
clearly we can reach $[p(x)]\downt$ or $[q(x)]\downt$. Since both
$[p(x)]\downt$ and  $[q(x)]\downt$ give better values, we can safely replace
$[q(y)]\downt$ with $0$, thus obtaining the final result Figure~\ref{Fig:R7order3}(d).
In theory we can generalize P7.1 as
$\B\models [\exists \vec{x}, \EF(e_2)]
\rightarrow [\exists \vec{y_1}, \EF(e_{11})]\vee \cdots \vee [\exists \vec{y_n}, \EF(e_{1n})]$
where $\vec{x}$ are the variables in $\EF(e_2)$ 
and $\vec{y_i}$ the variables in $\EF(e_{1i})$ where $1\leq i\leq n$,
and generalize the corresponding value condition V7.1 as
$\forall i\in [1,n],\min(target(e_{1i}))\geq\max(target(e_2))$.
We can generalize other reachability and value conditions similarly.
However the resulting conditions are too expensive to test in practice.

\begin{figure}[tb]
\begin{center}
\centerline{\psfig{figure=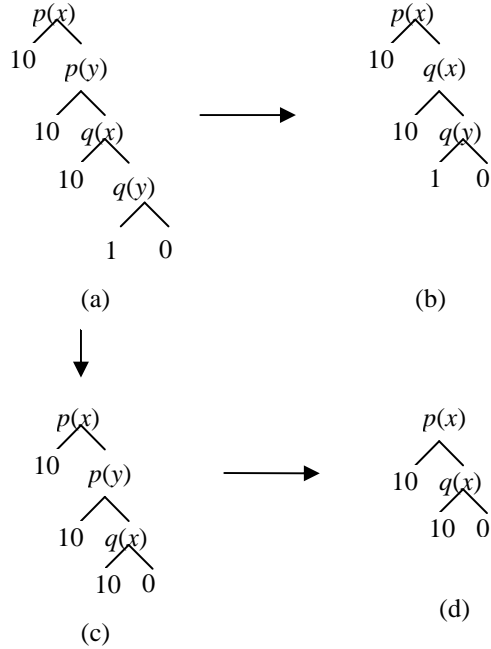}}
\caption{An example illustrating that the final result of R7 reductions
  is order dependent.}
\label{Fig:R7order3}
\end{center}
\vskip -0.2in
\end{figure}

\subsubsection{Relaxation of Reachability Conditions}

The conditions P7.2 and P7.3 are sufficient, but not necessary to guarantee correct reductions.
Sometimes valuations just need to agree on a smaller set of variables than the intersection
of variables. To see this,
consider the example as shown in  Figure~\ref{Fig:minsetvar}, where
$A\ominus B>0$ and the intersection is $\{x, y, z\}$.
However, to guarantee $A\ominus B>0$ we just need to agree on either $\{x,y\}$ or $\{x,z\}$.
Intuitively we have to agree on the variable $x$ to avoid the situation when
two paths $p(x,y) \wedge \neg q(x)$ and $p(x,y) \wedge q(x)\wedge h(z)$ can co-exist.
In order to prevent the co-existence of two paths  $\neg p(x,y) \wedge \neg h(z)$ and
$p(x,y) \wedge q(x)\wedge h(z)$, either $y$ or $z$ has to be the same as well.
Now if we change this example a little bit and
replace each  $h(z)$ with $h(z,v)$, then we have two minimal sets of variables
of different size, one is $\{x,y\}$, and the other is $\{x,z,v\}$.
As a result we cannot identify a minimum set of variables for the subtraction 
and must either choose the intersection or
heuristically identify a minimal set, for example, using a greedy procedure.

\begin{figure}[tbhp]
\begin{center}
\centerline{\psfig{figure=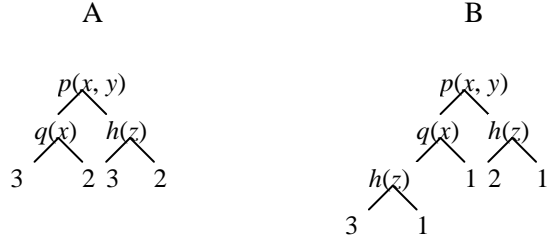}}
\caption{An example illustrating that the minimal set of variables for
  subtraction is not unique.}
\label{Fig:minsetvar}
\end{center}
\vskip -0.2in
\end{figure} 

\subsection{(R8) Weak Reduction by Unification}

Consider a FODD $B$. Let $\vec{v}$ denote its variables, and let
$\vec{x}$ and $\vec{y}$ be disjoint subsets of  $\vec{v}$, which 
are of the same cardinality.
We define the operator $\reightunify(B,\vec{x},\vec{y})$ as replacing
variables in $\vec{x}$ by the corresponding variables in $\vec{y}$.
We denote the resulting FODD by $B\{\vec{x}/\vec{y}\}$ so the result
has variables in $\vec{v}$ \texttt{\char92}$\vec{x}$.
We have the following condition for the correctness of R8:

{\cond {V8}:} all leaves in $B\{\vec{x}/\vec{y}\}\ominus B$ are non negative.

\begin{lemma}
Let $B$ be a FODD, 
$B'$ the result of $\reightunify(B,\vec{x},\vec{y})$ for which V8 holds, 
then for any interpretation $I$
we have $\map_B(I)=\map_{B'}(I)$. 
\end{lemma}

\begin{proof}
Consider any valuation $\zeta_1$ to $\vec{v}$ in $B$.
By V8, $B\{\vec{x}/\vec{y}\}$ gives a better value on the same valuation. Therefore
we do not lose any value by this operator.
We also do not gain any extra value. Consider any valuation $\zeta_2$ to
variables in $B'$ reaching a leaf node with value $v$,
we can construct a valuation $\zeta_3$ to $\vec{v}$ in $B$ with all variables
in $\vec{x}$ taking the corresponding value in $\vec{y}$, and it will reach 
a leaf node in $B$ with the same value.
Therefore the map will not be changed by unification.
\end{proof}

Figure~\ref{Fig:unification}
illustrates that in some cases R8 is applicable where R7 
is not.
We can apply R8 with $\{x_1/x_2\}$ 
to get a FODD as shown in Figure~\ref{Fig:unification}(b).
Since $(b)\ominus (a)\geq 0$, $(b)$ becomes the result 
after reduction. Note that if we unify in the other way, 
i.e.,$\{x_2/x_1\}$, we will get Figure~\ref{Fig:unification}(c), 
it is isomorphic to Figure~\ref{Fig:unification}(b), but we cannot reduce
the original FODD to this result, because $(c)\ominus (a)\not\geq 0$.
This phenomenon happens since 
the subtraction operation (implemented by Apply) used in 
the reductions is  propositional 
and therefore sensitive to variable names.

\begin{figure}[tbhp]
\vskip 0.2in
\begin{center}
\centerline{\psfig{figure=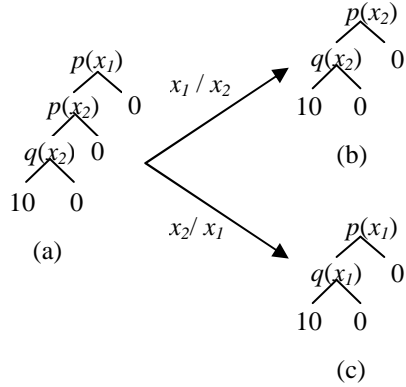}}
\caption{An example illustrating R8.}
\label{Fig:unification}
\end{center}
\vskip -0.2in
\end{figure} 

\subsection{(R9) Equality Reduction}
\label{Sec:equality}

Consider a FODD $B$ with an equality node $n$ labeled $t=x$.
Sometimes we can drop $n$ and connect its parents to
a sub-FODD that is the result of taking the maximum of the left and the right children of $n$. 
For this reduction to be applicable $B$ has to satisfy the following condition.

{\cond {E9.1}:}
For an equality node $n$ labeled $t=x$ at least one of $t$ and $x$ is a variable 
and it appears neither in $n\downf$ nor in the node formula for $n$. 
To simplify the description of the reduction procedure below, we assume that $x$ is that variable.

Additionally we make the following assumption about the domain.

{\cond {D9.1}:} 
The domain contains more than one object.

The above assumption guarantees that valuations reaching the right child of equality nodes exist. 
This fact is needed in proving correctness of the Equality reduction operator. 
First we describe the reduction procedure for $\rnine(n)$. 
Let $B_n$ denote the FODD rooted at node $n$ in FODD $B$. 
We extract a copy of $B_{n\downt}$ (and name it $B_{n\downt}\mbox{-copy}$), 
and a copy of $B_{n\downf}$ ($B_{n\downf}\mbox{-copy}$) from $B$. 
In $B_{n\downt}\mbox{-copy}$, we rename the variable $x$ to $t$ to 
produce diagram $B'_{n\downt}\mbox{-copy}$. 
Let $B_n'=Apply(B'_{n\downt}\mbox{-copy}, B_{n\downf}\mbox{-copy}, max)$. 
Finally we drop the node $n$ in $B$ and 
connect its parents to the root of $B_n'$ to obtain the final result $B'$.
An example is shown in Figure~\ref{Fig:eqRed1}.

Informally, we are extracting the parts of the FODD rooted at node $n$,
one where $x = t$ (and renaming $x$ to $t$ in that part) and 
one where $x \neq t$. 
The condition E9.1 and the assumption D9.1 guarantee that 
regardless of the value of $t$, we have valuations reaching both parts. 
Since by the definition of $\map$, we maximize over the valuations, 
in this case we can maximize over the diagram structure itself.
We do this by calculating the function which is the maximum of the two
functions corresponding to the two children of $n$ (using $Apply$) 
and 
replacing the old sub-diagram rooted at node $n$ by the new combined diagram. 
Theorem \ref{EqRedThm} proves that this does not affect the map of $B$.

One concern for implementation is that we simply replace the old sub-diagram 
by the new sub-diagram, which may result in a diagram where strong
reductions are applicable.
While this is not a problem semantically, 
we can avoid the need for strong reductions by using $Apply$ 
that implicitly performs strong reductions R1(neglect) and R2(join) as follows.

Let  $B_a$ denote the FODD resulting from replacing node $n$ in $B$ with $0$, and 
     $B_b$ the FODD resulting from replacing node $n$ with $1$ and all leaves other than node $n$ by $0$, 
we have  the final result $B' = B_a \oplus B_b'$ where $B_b'=B_b \otimes B_n'$. 
By correctness of $Apply$ the two forms of calculating $B'$
give the same map.

\begin{figure}[tbhp]
\vskip 0.2in
\begin{center}
\centerline{\psfig{figure=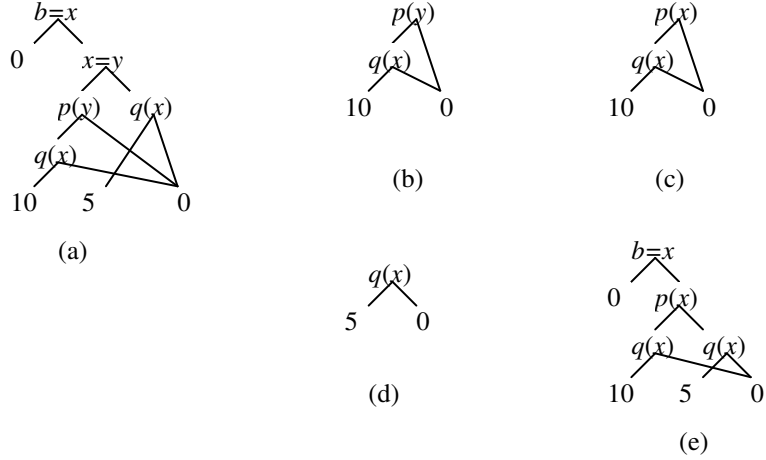}}
\caption{An example of the equality reduction.
(a) The FODD before reduction. The node $x = y$ satisfies condition E9.1 for variable $y$. 
(b) $B_{n\downt}\mbox{-copy}$ ($n\downt$ extracted).
(c) $B_{n\downt}\mbox{-copy}$ renamed to produce $B'_{n\downt}\mbox{-copy}$.
(d) $B_{n\downf}\mbox{-copy}$.
(e) Final result with node $n$ replaced by $apply(B'_{n\downt}\mbox{-copy}, B_{n\downf}\mbox{-copy}, max)$}
\label{Fig:eqRed1}
\end{center}
\vskip -0.2in
\end{figure}

In the following we prove that for any node $n$ 
where equality condition E9.1 holds in $B$ we can perform 
the equality reduction $\rnine$ without changing the map for any
interpretation satisfying D9.1.
We start with properties of FODDs defined above, 
e.g., $B_a$, $B_b$, and $B_b'$. 
Let $\Gamma_n$ denote the set of all valuations reaching node $n$ and 
let $\Gamma_m$ denote the set of all valuations not reaching node $n$
in B. 
From the basic definition of $\map$ we have the following:

\begin{claim}
\label{Cl:BaBb}
For any interpretation $I$, \\
(a) $\forall$ $\zeta \in \Gamma_m$, $\map_{B_a}(I, \zeta) = \map_B(I, \zeta)$.\\
(b) $\forall$ $\zeta \in \Gamma_n$, $\map_{B_a}(I, \zeta) = 0$.\\
(c) $\forall$ $\zeta \in \Gamma_m$, $\map_{B_b}(I, \zeta) = 0$.\\
(d) $\forall$ $\zeta \in \Gamma_n$, $\map_{B_b}(I, \zeta) = 1$.
\end{claim}
From Claim~\ref{Cl:BaBb} and the definition of $\map$, we have,

\begin{claim}
\label{Cl:B'bB'n}
For any interpretation $I$, \\
(a) $\forall$ $\zeta \in \Gamma_m$,
$\map_{B_b'}(I, \zeta) = 0$.\\
(b) $\forall$ $\zeta \in \Gamma_n$,
$\map_{B_b'}(I, \zeta) = \map_{B'_n}(I, \zeta)$.
\end{claim}
From Claim~\ref{Cl:BaBb}, Claim~\ref{Cl:B'bB'n}, and the definition of $\map$ we have,

\begin{claim}
\label{Cl:B'B'n}
For any interpretation $I$,\\
(a) $\forall$ $\zeta \in \Gamma_m$, $\map_{B'}(I, \zeta) = \map_B(I, \zeta)$.\\
(b) $\forall$ $\zeta \in \Gamma_n$, $\map_{B'}(I, \zeta) = \map_{B'_n}(I, \zeta)$.
\end{claim}

Next we prove the main property of this reduction stating 
that for all valuations reaching node $n$ in $B$, 
the old sub-FODD rooted at $n$ and the new (combined) sub-FODD produce 
the same map.

\begin{lemma}
\label{Lem:B'n}
Let $\Gamma_n$ be the set of valuations reaching node $n$ in FODD $B$. 
For any interpretation $I$ satisfying D9.1, 
$\max_{\zeta \in \Gamma_n} \map_{B_n}(I, \zeta) = \max_{\zeta \in \Gamma_n} \map_{B'_n}(I, \zeta)$.
\end{lemma}

\begin{proof}
By condition E9.1, the variable $x$ does not appear in $NF(n)$ 
and hence its value in $\zeta \in \Gamma_n$ is not constrained. 
We can therefore partition the valuations in $\Gamma_n$ into disjoint sets, 
$\Gamma_n = \{\Gamma_{\Delta} \mid \Delta \mbox{ is a valuation to variables other than } x\}$, 
where in $\Gamma_{\Delta}$ variables other than $x$ are fixed to their value 
in $\Delta$ and $x$ can take any value in the domain of $I$. 
Assumption D9.1 guarantees that every $\Gamma_{\Delta}$ 
contains at least one valuation reaching $B_{n\downt}$ and 
at least one valuation reaching $B_{n\downf}$ in $B$. 
Note that if a valuation $\zeta$ reaches $B_{n\downt}$
then $t=x$ is satisfied by $\zeta$ thus 
$\map_{B_{n\downt}}(I,\zeta)=\map_{B'_{n\downt}\mbox{-copy}}(I,\zeta)$.
Since $x$ does not appear in $B_{n\downf}$
we also have that
$\map_{B'_{n\downf}\mbox{-copy}}(I,\zeta)$
is constant for all $\zeta \in \Gamma_{\Delta}$.
Therefore by the correctness of $Apply$ we have
$\max_{\zeta \in \Gamma_{\Delta}}\map_{B_n}(I, \zeta) = \max_{\zeta \in \Gamma_{\Delta}}\map_{B_n'}(I, \zeta)$.

Finally, by the definition of $\map$,
$\max_{\zeta \in \Gamma_n} \map_{B_n}(I, \zeta) = \max_{\Delta}\max_{\zeta \in \Gamma_{\Delta}} \map_{B_n}(I, \zeta)$ $= \max_{\Delta}\max_{\zeta \in \Gamma_{\Delta}} \map_{B_n'}(I, \zeta) = \max_{\zeta \in \Gamma_n} \map_{B_n}(I, \zeta)$.
\end{proof}

\begin{lemma}
\label{EqRedThm}
Let $B$ be a FODD, $n$ a node for which condition E9.1
holds, and $B'$ be the result of $\rnine(n)$, then
for any interpretation $I$ satisfying D9.1, $\map_B(I) = \map_{B'}(I)$.
\end{lemma}

\begin{proof}
Let $X = \max_{\zeta \in \Gamma_m} \map_{B'}(I, \zeta)$ and
$Y = \max_{\zeta \in \Gamma_n} \map_{B'}(I, \zeta)$. 
By the definition of $\map$, $\map_{B'}(I) = max(X, Y)$. 
However, by Claim~\ref{Cl:B'B'n}, 
$X = \max_{\zeta \in \Gamma_m} \map_B(I, \zeta)$ 
and by Claim~\ref{Cl:B'B'n} and Lemma~\ref{Lem:B'n}, 
$Y = \max_{\zeta \in \Gamma_n} \map_{B'_n}(I, \zeta) = \max_{\zeta \in \Gamma_n} \map_{B_n}(I, \zeta)$. 
Thus $max(X, Y) = \map_B(I) = \map_{B'}(I)$.
\end{proof}

While Lemma~\ref{EqRedThm} guarantees correctness, 
when applying it in practice 
it may be important to avoid violations of the sorting order 
(which would require expensive re-sorting of the diagram). 
If both $x$ and $t$ are variables we can sometimes replace both with a new
variable name so the resulting diagram is sorted. 
However this is not always possible.
When such a violation is unavoidable, 
there is a tradeoff between performing the reduction and 
sorting the diagram and ignoring the potential reduction. 

To summarize,
this section introduced several new reductions that can compress
diagrams significantly.
The first (R5) is 
a generic strong reduction that removes
implied branches in a diagram.
The other three (R7, R8, R9) are weak reductions that do not alter the
overall map
of the diagram but do alter the map for specific valuations.
The three reductions are complementary since they capture different
opportunities for space saving.

\section{Decision Diagrams for MDPs}
\label{Sec:relationalRep}

In this section we show how FODDs can be used to capture a RMDP.
We
therefore use FODDs to represent 
the domain dynamics of deterministic action alternatives,
the probabilistic choice of action alternatives, the reward function, and value functions. 

\subsection{Example Domain} 

We first give a concrete formulation of the logistics problem 
discussed in the introduction. This example follows exactly the details given by
\citeA{BoutilierRePr2001}, and is used  to illustrate our constructions for MDPs. The
domain includes boxes, trucks and cities, and predicates are
$Bin(Box,City)$, $Tin(Truck,City)$, and $On(Box,Truck)$. 
Following \citeA{BoutilierRePr2001},
we assume that $On(b,t)$ and $Bin(b,c)$ are mutually exclusive,
so a box on a truck is not in a city and vice versa.
That is, our background knowledge includes  statements
$\forall b,c,t, On(b,t)\rightarrow \neg Bin(b,c)$ and
$\forall b,c,t, Bin(b,c)\rightarrow \neg On(b,t)$.
The reward function, capturing a planning goal,
awards a reward of 10 if the formula $\exists b, Bin(b,Paris)$ is
true, that is if there is any box in Paris. Thus the reward is allowed
to include constants but need not be completely ground.  

The domain includes 3 actions 
$load, unload$, and $drive$.
Actions have no effect if their
preconditions are not met. Actions can also fail with some
probability. 
When attempting $load$, a successful version $loadS$ is executed with
probability 0.99, and an unsuccessful version $loadF$ (effectively a
no-operation) with probability 0.01. 
The drive action is executed
deterministically. 
When attempting $unload$, 
the probabilities depend on whether it is raining or not. 
If it is not raining then 
a successful version $unloadS$ is executed with
probability 0.9, and $unloadF$ 
with probability 0.1.
If it is raining 
$unloadS$ is executed with
probability 0.7, and $unloadF$ 
with probability 0.3.

\subsection{The Domain Dynamics} 
\label{Sec:DomainDynamics}

We follow \citeA{BoutilierRePr2001} and specify stochastic actions as a
randomized choice among deterministic alternatives. 
The domain dynamics
are defined by {\em truth value diagrams}
(TVDs). For every action schema $A(\vec{a})$ and each predicate schema
$p(\vec{x})$ the TVD $T(A(\vec{a}),p(\vec{x}))$ is a FODD with
$\{0,1\}$ leaves. 
The TVD gives the truth value of $p(\vec{x})$ 
in the next state when $A(\vec{a})$ has been performed in the current
state.  
We call $\vec{a}$ action
parameters, and $\vec{x}$ predicate parameters.
No other variables are allowed in the TVD; 
the reasoning behind this restriction is explained in 
Section~\ref{Sec:RegressionProb}.
The restriction can be sometimes sidestepped by introducing
more action parameters instead of the variables.  

The truth value of a TVD is valid when we
fix a valuation of the parameters. 
The TVD simultaneously captures the truth values of all
instances of $p(\vec{x})$ in the next state. Notice that
TVDs for different predicates are separate. This can
be safely done even if an action has coordinated effects (not
conditionally independent) since the action alternatives  are deterministic. 

Since we allow both action parameters and predicate
parameters, the effects of an action are not restricted to predicates
over action arguments so TVD are more expressive than simple STRIPS
based schemas. 
For example, TVDs can easily express universal effects of an action.
To see this note that if $p(\vec{x})$ is true for all $\vec{x}$ after action
$A(\vec{a})$ then the TVD 
$T(A(\vec{a}),p(\vec{x}))$ can be captured by a leaf valued 1.
Other universal conditional effects can be captured similarly.
On the other hand, since we do not have explicit universal
quantifiers, TVDs cannot capture universal preconditions. 

\begin{figure}[tbhp]
\vskip 0.2in
\begin{center}
\centerline{\psfig{figure=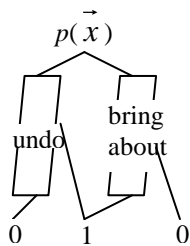}}
\caption{A template for the TVD}
\label{Fig:basictvd}
\end{center}
\vskip -0.2in
\end{figure}

\begin{figure}[tbhp]
\vskip 0.2in
\begin{center}
\centerline{\psfig{figure=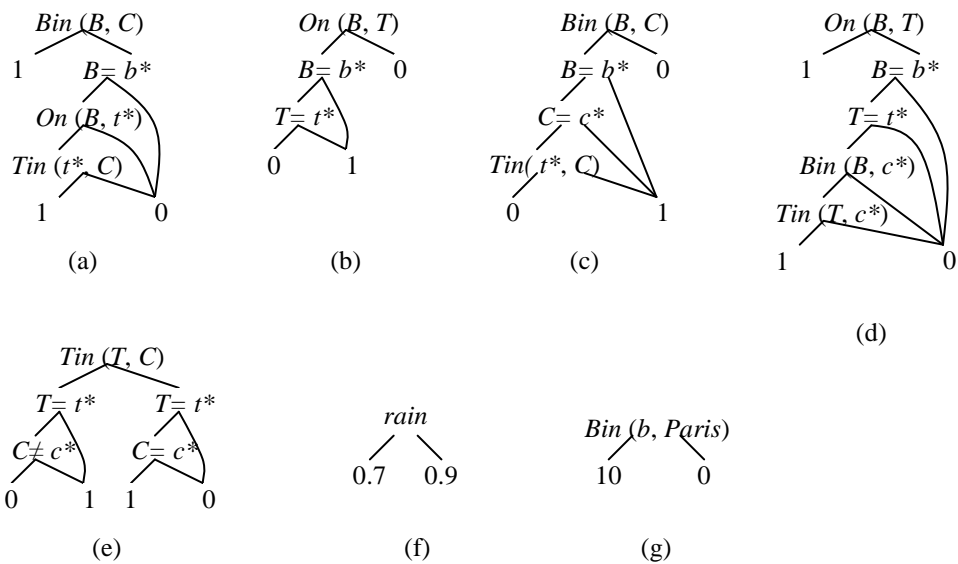}}
\caption{FODDs for logistics domain: TVDs, action choice, and reward function.
(a)(b) The TVDs for $Bin(B, C)$ and $On(B, T)$ under action choice $unloadS(b^*, t^*)$.
(c)(d) The TVDs for $Bin(B, C)$ and $On(B, T)$ under action choice
  $loadS(b^*, t^*, c^*)$. Note that $c^*$ must be an action
  parameter so that (d) is a valid TVD.
(e) The TVD for $Tin(T, C)$ under action choice $driveS(t^*, c^*)$.
(f) The probability FODD for the action choice $unloadS(b^*, t^*)$.
(g) The reward function. }
\label{Fig:logisticsFODD}
\end{center}
\vskip -0.2in
\end{figure}

For any domain, a TVD for predicate $p(\vec{x})$ can be defined generically
as in Figure~\ref{Fig:basictvd}. The idea is that the predicate is true
if it was true before and is not ``undone'' by the action or was false
before and is ``brought about'' by the action. 
TVDs for the logistics domain in our running example are given in 
Figure~\ref{Fig:logisticsFODD}. 
All the TVDs omitted 
in the figure
are trivial
in the sense that the predicate is not affected by the action.
In order to simplify the presentation we give the TVDs in their
generic form and did not sort the diagrams using the order proposed in
Section~\ref{Sec:labelorder};
the TVDs are consistent with the ordering 
$Bin \prec $  ``=''  $\prec On \prec Tin \prec rain$. 
Notice that the TVDs capture the implicit assumption
usually taken in such planning-based domains that if the preconditions
of the action are not satisfied then the action has no effect. 

Notice how we utilize the multiple path semantics with maximum
aggregation. A predicate is true if it is true according to one of the paths
specified so we get a disjunction over the conditions for free. 
If we use the single path semantics 
of \citeA{BlockeelDR98}
the corresponding notion of TVD is
significantly more complicated since a single path must capture all
possibilities for a predicate to become true. To capture that, 
we must test
sequentially for different conditions and then take a union of
the substitutions from different tests and in turn 
this requires additional
annotation on FODDs with appropriate semantics. 
Similarly an OR operation would require union of substitutions, thus
complicating the representation. 
We explain these issues in more detail in Section~\ref{Sec:explainSinglePath}
after we introduce the first order value iteration algorithm.

\subsection{Probabilistic Action Choice} 
One can consider modeling arbitrary conditions described by 
formulas over the state to control nature's probabilistic choice of
action. 
Here the multiple path semantics makes it hard to specify mutually
exclusive conditions using existentially quantified variables
and in this way specify a distribution.
We therefore
restrict the conditions to be either propositional or
depend directly on the action parameters. 
Under this condition any interpretation follows exactly one path
(since there are no variables and thus only the empty valuation)
thus the aggregation function does not interact with the probabilities assigned.
A diagram showing action choice for $unloadS$ in 
our logistics example is given in Figure~\ref{Fig:logisticsFODD}.
In this example, the condition is propositional. The condition can also depend on
action parameters, for example, if we assume
that the result is also affected by whether the box is big or not, we can
have a diagram as in Figure~\ref{Fig:bigprob} specifying the
action choice probability.

\begin{figure}[tbhp]
\vskip 0.2in
\begin{center}
\centerline{\psfig{figure=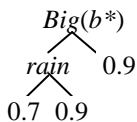}}
\caption{An example showing that the choice probability can depend on action parameters.}
\label{Fig:bigprob}
\end{center}
\vskip -0.2in
\end{figure}

Note that a probability usually 
depends on the current state.
It can depend on arbitrary properties of the state (with the restriction
stated as above), e.g., $rain$ and
$big(b^*)$, as shown in Figure~\ref{Fig:bigprob}. 
We allow arbitrary conditions that depend on predicates with arguments restricted to action parameters
so the dependence can be complex. However, we do not 
allow any free variables in the probability choice diagram. For example,
we cannot model a probabilistic choice of $unloadS(b^*,t^*)$ that depends on other
boxes on the truck $t^*$, e.g.,
$\exists b, On(b,t^*)\wedge b\neq b^*:0.2$;
otherwise, $0.7$.
While we can write a FODD to capture
this condition, the semantics of FODD means that a path to $0.7$ will be selected by
max aggregation so the distribution cannot be modeled  
in this way. While this is clearly a restriction, the conditions based on action
arguments
still give a substantial modeling power.

\subsection{Reward and Value Functions}
Reward and value functions
can be represented directly
using algebraic FODDs. 
The reward function for our logistics domain example is given in 
Figure~\ref{Fig:logisticsFODD}.

\section{Value Iteration with FODDs}
\label{Sec:VI}

Following \citeA{BoutilierRePr2001}
we define the first order value iteration algorithm 
as follows: given
the reward function $R$ and the action model as input, we set $V_0 =
R, n=0$ and repeat the procedure {\em Rel-greedy} until termination:

\begin{procedure}Rel-greedy\\
\label{Proc:relgreedy}
1. For each action type $A(\vec{x})$, compute:
\begin{equation}
\label{Eq:Qfunction}
Q_{V_n}^{A(\vec{x})} = R \oplus[\gamma \otimes \oplus_j(prob(A_j(\vec{x})) \otimes
Regr(V_n, A_j(\vec{x})))]
\end{equation}
2. 
$Q_{V_n}^A = \mbox{obj-max} (Q_{V_n}^{A(\vec{x})})$.
\\
3.
$V_{n+1} = \max_A Q_{V_n}^A$. 
\end{procedure}

The notation and steps of this procedure were discussed in
Section~\ref{Sec:introRMDP} except that now $\otimes$
and $\oplus$ work on FODDs instead of case statements.
Note that since the reward function
does not depend on actions, we can
move the  object maximization  step forward
before adding the reward function. 
I.e., we first have 
$$T_{V_n}^{A(\vec{x})} = \oplus_j(prob(A_j(\vec{x})) \otimes
Regr(V_n, A_j(\vec{x}))),$$
followed by
$$Q_{V_n}^A =R \oplus \gamma \otimes \mbox{obj-max} (T_{V_n}^{A(\vec{x})}).$$
Later we will see that the object maximization step makes  
more reductions possible; therefore by moving this step forward we get some savings in
computation. We compute the updated 
value function in this way in the comprehensive example of value iteration 
given later in Section~\ref{Sec:VIexample}. 

Value iteration terminates when $\|V_{i+1} - V_i\| \leq
\frac{\varepsilon (1 - \gamma)}{2\gamma}$ \cite{Puterman1994}. 
In our case we need to test that the values achieved by the two
diagrams is within $\frac{\varepsilon (1 - \gamma)}{2\gamma}$.

Some formulations of goal based planning problems 
use an absorbing state with zero additional reward 
once the goal is reached.
We can handle this formulation when there is
only one non-zero leaf in $R$.
In this case, we can replace Equation~\ref{Eq:Qfunction}
with $$Q_{V_n}^{A(\vec{x})} = max (R,\gamma \otimes \oplus_j(prob(A_j(\vec{x})) \otimes
Regr(V_n, A_j(\vec{x}))).$$
To see why this is correct,
note that due to discounting the max value is always $\leq R$.
If $R$ is satisfied in a state we do not care about the action (max would
be $R$) and if $R$ is $0$ in a state we get the value of the discounted future reward.

Note that we can only do this in  goal based domains, i.e., there is only one 
non-zero leaf. This does not mean that we cannot have disjunctive goals, but it means 
that we must value each goal condition equally. 

\subsection{Regressing Deterministic Action Alternatives} 
\label{Sec:RegressionbyBR}
We first describe the calculation of
$Regr(V_n, A_j(\vec{x}))$ using a simple idea we call block
replacement. We then proceed to discuss how to obtain the result efficiently.

Consider $V_n$ and the nodes in its FODD. For each
such node take a copy of the corresponding TVD, where
predicate parameters are renamed so that they correspond to the node's
arguments and action parameters are unmodified.
$\mbox{BR-regress}(V_n,A(\vec{x}))$ is the FODD resulting from replacing each
node in $V_n$ with the corresponding TVD,
with outgoing edges connected to the 0, 1 leaves of the TVD.

Recall that a RMDP represents a family of concrete MDPs each generated
by choosing a concrete instantiation of the state space (typically
represented by the number of objects and their types).
The formal properties of our algorithms hold for any concrete
instantiation. 

Fix any concrete instantiation of the state space.
Let $s$ denote a state resulting from executing an action $A(\vec{x})$
in state $\hat{s}$. 
Notice that $V_n$ and $\mbox{BR-regress}(V_n,A(\vec{x}))$ have exactly the same
variables. We have the following lemma:
\begin{lemma}
\label{Lem:DetRegressA}
Let $\zeta$ be any valuation to the variables of $V_n$ (and thus also
the variables of $\mbox{BR-regress}(V_n,A(\vec{x}))$). Then
$\map_{V_n}(s,\zeta) = \map_{BR-regress(V_n,A(\vec{x}))}(\hat{s},\zeta)$.
\end{lemma} 
\begin{proof}
Consider the paths $P,\hat{P}$ followed under the valuation $\zeta$ in the two diagrams. 
By the definition of TVDs, the sub-paths
of $\hat{P}$ applied to $\hat{s}$ guarantee that the corresponding
nodes in $P$ take the same truth values in $s$. 
So $P,\hat{P}$ reach the same leaf and the same value is obtained.
\end{proof}

A naive implementation of block replacement may not be efficient.
If we use block replacement for regression then the resulting
FODD is not necessarily reduced and moreover, since the different
blocks are sorted to start with the result is not even sorted. 
Reducing and sorting the results may be an expensive
operation. Instead we calculate the result 
as follows. 
For any FODD $V_n$ we traverse 
$\mbox{BR-regress}(V_n,A(\vec{x}))$ 
using postorder traversal
in terms of blocks and combine
the
blocks. At any step we have to combine up to 3 FODDs such
that the parent block has not yet been processed (so it is a TVD with
binary leaves) and the two children have been processed (so they are
general FODDs). If we call the parent $B_n$, the \true branch child
$B_t$ and the \false branch child $B_f$ then we can represent their
combination as $[B_n\otimes B_t] \oplus [(1 \ominus B_n)\otimes B_f]$.

\begin{lemma}
\label{Lem:DetRegress}
Let $B$ be a FODD where 
$B_t$ and $B_f$ are FODDs, and $B_n$ is a FODD with $\{0,1\}$ leaves.
Let $\hat{B}$ be the result of using Apply to calculate
the diagram 
 $[B_n\otimes B_t] \oplus [(1 \ominus B_n)\otimes B_f]$.
Then for any interpretation $I$ and valuation $\zeta$
we have $\map_B(I,\zeta)=\map_{\hat{B}}(I,\zeta)$.
\end{lemma}
\begin{proof}
This is true since 
by fixing the valuation we effectively ground the FODD
and all paths are mutually exclusive. 
In other words the FODD becomes propositional and clearly the combination
using propositional Apply is correct.
\end{proof}

A high-level description of the algorithm 
to calculate $\mbox{BR-regress}(V_n,A(\vec{x}))$ by block combination 
is as follows:

\begin{procedure}Block Combination for $\mbox{BR-regress}(V_n,A(\vec{x}))$
\begin{enumerate}
\item
Perform a topological sort on $V_n$ nodes \cite<see for example>{CormenLeRiSt2001}.
\item
In reverse order, for each non-leaf node $n$ (its children 
$B_t$ and $B_f$ have already been processed), let $B_n$ be a copy of the 
corresponding TVD, calculate
$[B_n\otimes B_t] \oplus [(1 \ominus B_n)\otimes B_f]$.
\item
Return the FODD corresponding to the root.
\end{enumerate}
\end{procedure}

Notice that different blocks share variables so we cannot perform weak
reductions during this process. However, we can perform strong
reductions in intermediate steps since they do not change the map for
any valuation. After the process is completed we can perform any
combination of weak and strong reductions since this does not change
the map of the regressed value function.

\begin{figure}[tbhp]
\vskip 0.2in
\begin{center}
\centerline{\psfig{figure=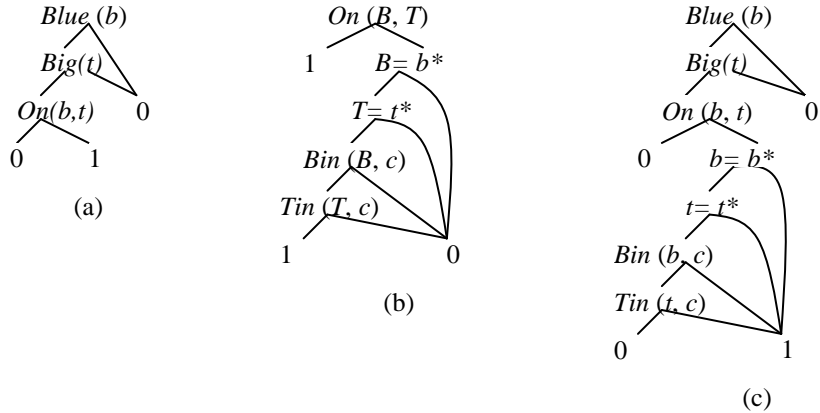}}
\caption{An example illustrating why variables are not allowed in TVDs.}
\label{Fig:novarinTVD}
\end{center}
\vskip -0.2in
\end{figure}

We can now explain  why we cannot have variables in TVDs
through an example illustrated in Figure~\ref{Fig:novarinTVD}.
Suppose we have a value function as defined in Figure~\ref{Fig:novarinTVD}(a), 
saying that if there is a blue block and a big truck 
such that the block is not on the truck
then 
 value $1$ is assigned. 
Figure~\ref{Fig:novarinTVD}(b) gives
the TVD for $On(B,T)$ under action $loadS$,
in which $c$ is a variable
 instead of an action parameter.
Figure~\ref{Fig:novarinTVD}(c) gives the result 
after block replacement.
Consider an interpretation $\hat{s}$ with domain $\{b_1,t_1,c_1, c_2\}$ and relations 
$\{Blue(b_1),Big(t_1),Bin(b_1,c_1),Tin(t_1,c_1)\}$.
After the action $loadS(b_1, t_1)$ we will reach the state
$s=\{Blue(b_1),Big(t_1),On(b_1,t_1),Tin(t_1,c_1)\}$, which 
gives us a value of $0$.
But Figure~\ref{Fig:novarinTVD}(c) with $b^*=b_1,t^*=t_1$ evaluated in
$\hat{s}$ gives value of 1 by valuation $\{b/b_1,c/c_2,t/t_1\}$.
Here the choice $c/c_2$ makes sure the precondition is violated.
By making $c$ an action parameter, applying the action must explicitly
choose a  valuation and this leads to a correct value function.
Object maximization turns action parameters into variables and
allows us to choose the argument so as to maximize the value.

\subsection{Regressing Probabilistic Actions} 
\label{Sec:RegressionProb}

To regress a probabilistic action we must regress all its
deterministic alternatives and combine each with its choice
probability as in Equation~\ref{Eq:Qfunction}.
As discussed in Section~\ref{Sec:introRMDP}, due to the restriction in
the RMDP model that explicitly specifies a finite number of
deterministic action alternatives, we can replace the potentially
infinite sum of Equation~\ref{Eq:VI} with the finite sum of 
Equation~\ref{Eq:Qfunction}.
If this is done correctly for every state then the result of 
Equation~\ref{Eq:Qfunction} is correct.
In the
following we specify how this can be done with FODDs.

Recall that $prob(A_j(\vec{x}))$ is restricted to include 
only action parameters and cannot include variables.
We can therefore calculate $prob(A_j(\vec{x})) \otimes
Regr(V_n, A_j(\vec{x}))$ in step (1) directly using Apply.
However, the different regression results are independent functions so
that in the sum 
$\oplus_j(prob(A_j(\vec{x})) \otimes Regr(V_n, A_j(\vec{x})))$
we must standardize apart the different regression results
before adding the functions (note that action parameters are still
considered constants at this stage).
The same holds for the addition of the reward function.
The need to standardize apart complicates the diagrams and often
introduces structure that can be reduced. 
When performing these operations we first use the propositional Apply procedure
and then follow with weak and strong reductions.

\begin{figure}[tbhp]
\vskip 0.2in
\begin{center}
\centerline{\psfig{figure=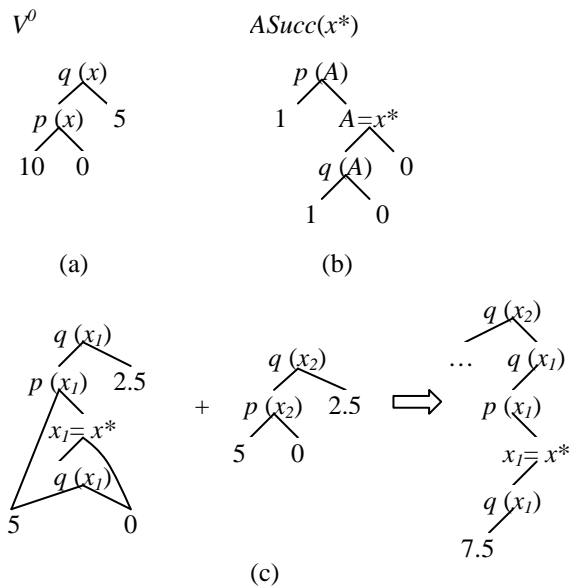}}
\caption{An example illustrating the need to standardize apart.}
\label{Fig:stdapart}
\end{center}
\vskip -0.2in
\end{figure} 

Figure~\ref{Fig:stdapart}
illustrates why we need to standardize apart different action outcomes.
Action $A$ can succeed (denoted as $ASucc$) or 
fail (denoted as $AFail$, effectively a no-operation),
and each is chosen with probability 0.5. 
Part (a) gives the value function $V^0$.
Part (b) gives the TVD for $P(A)$ under the action choice $ASucc(x^*)$.
All other TVDs are trivial. 
Part (c) shows part of the result of adding the two outcomes for 
$A$ after standardizing apart 
(to simplify the presentation the diagrams are not sorted).
Consider an interpretation with domain $\{1,2\}$ and relations $\{q(1),p(2)\}$.
As can be seen from (c), by choosing $x^* = 1$, i.e.\ action $A(1)$, 
the valuation $x_1 = 1, x_2 = 2$ gives 
a value of $7.5$ after the action
(without considering the discount factor).
Obviously if we do not standardize apart (i.e $x_1=x_2$), 
there is no leaf with value $7.5$ and we get a wrong value.
Intuitively the contribution of $ASucc$ to the value comes from the
``bring about'' portion of the diagram and $AFail$'s contribution
uses bindings from
the ``not undo'' portion and the two portions can refer to
different objects. Standardizing apart allows us to capture
both simultaneously.

From Lemma~\ref{Lem:DetRegressA} and~\ref{Lem:DetRegress}
and the discussion so far we have:
\begin{lemma}
Consider any concrete instantiation of a RMDP. Let
$V_n$ be a value function for the corresponding MDP,
and let 
$A(\vec{x})$ be a probabilistic action in the domain.
Then $Q_{V_n}^{A(\vec{x})}$ as calculated by
Equation~\ref{Eq:Qfunction} is correct.
That is, for any state $s$, $\map_{Q_{V_n}^{A(\vec{x})}}(s)$ is 
the expected value of 
executing   $A(\vec{x})$ in $s$ and then receiving the terminal value $V_n$.
\end{lemma}

\subsection{Observations for Single Path Semantics}
\label{Sec:explainSinglePath}

Section~\ref{Sec:DomainDynamics} suggested that the single
path semantics 
of \citeA{BlockeelDR98}
does not support value iteration as well as the multiple path semantics.
Now with the explanation of regression,   
we can use an example to illustrate this.  Suppose we have a
value function as defined in Figure~\ref{Fig:union}(a), 
saying that if we have a red block in a big city then value $1$ is
assigned. Figure~\ref{Fig:union}(b) gives the result after block
replacement under action $unloadS(b^*,t^*)$. 
However this is not correct. 
Consider an interpretation $\hat{s}$ with domain 
$\{b_1,b_2,t_1,c_1\}$ and relations 
$\{Red(b_2),Blue(b_1),Big(c_1),Bin(b_1,c_1),Tin(t_1,c_1),On(b_2,t_1)\}$.
Note that we use the single path semantics. 
We follow the \true branch at the root since
$\exists b,c,Bin(b,c)$ is true with $\{b/b_1,c/c_1\}$. But
we follow the \false branch at $Red(b)$ since 
$\exists b,c,Bin(b,c)\wedge Red(b)$ is not satisfied. Therefore we get a value of $0$. 
Clearly, we should get a value of $1$  instead
with $\{b/b_2,c/c_1\}$, but it is impossible
to achieve this value in Figure~\ref{Fig:union}(b) with the single path semantics.
The reason block replacement fails is that  the top node decides 
on the true branch based on one instance of the predicate but we really need all true instances
of the predicate to filter into the true leaf of the TVD.

To correct the problem, we want to capture all instances that were true before and not undone
and all instances that are made true on one path.
Figure~\ref{Fig:union}(c) gives one possible way to do it.
Here $\leftarrow$ means variable renaming, and 
$\cup$ stands for union operator, which takes a union of
all substitutions. Both can be treated as edge operations.
Note that $\cup$ is a coordinated operation, i.e., instead of taking the
union of the substitutions for $b'$ and $b''$, $c'$ and $c''$ separately
we need to take the union of the substitutions for $(b',c')$ and $(b'',c'')$.
This approach may be possible but it clearly leads to complicated diagrams.
Similar complications arise in the context of object maximization.
Finally if we are to use this representation then all our procedures
will need to handle edge marking and unions of substitutions so this approach
does not look promising.

\begin{figure}[tbhp]
\vskip 0.2in
\begin{center}
\centerline{\psfig{figure=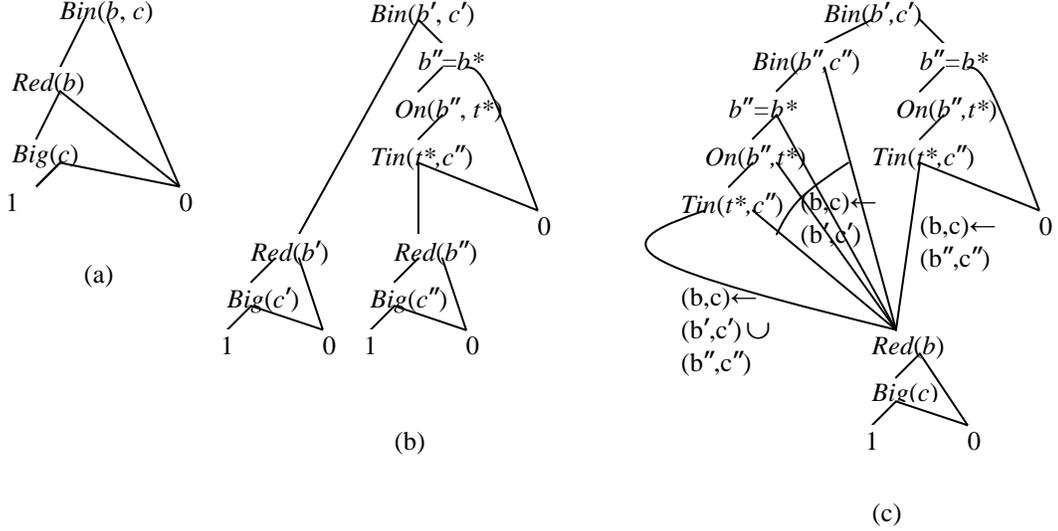}}
\caption{An example illustrating union or.}
\label{Fig:union}
\end{center}
\vskip -0.2in
\end{figure}

\subsection{Object Maximization} 
Notice that since we are handling different probabilistic alternatives of the 
same action separately we must keep action parameters fixed
during the regression process and until they are added in step $1$ of the algorithm.
In step 2 we maximize over the choice of action parameters.
As mentioned above we get this
maximization for free. We simply rename the
action parameters using new variable names (to avoid repetition
between iterations) and consider them as variables. The aggregation
semantics provides the maximization and by definition this
selects the best instance of the action. Since constants are turned into
variables
additional reduction is typically possible at this stage.
Any combination of weak and strong reductions can be used.
From the discussion we have the following lemma:
\begin{lemma}
Consider any concrete instantiation of a RMDP. Let
$V_n$ be a value function for the corresponding MDP,
and let 
$A(\vec{x})$ be a probabilistic action in the domain.
Then $Q_{V_n}^{A}$ as calculated by object maximization 
in step 2 of the algorithm is correct.
That is, for any state $s$, $\map_{Q_{V_n}^{A}}(s)$ is 
the maximum over expected values achievable by executing an instance of 
$A(\vec{x})$ in $s$ and then receiving the terminal value $V_n$.
\end{lemma}

A potential criticism of our object maximization is
that we are essentially adding more variables to the diagram and thus
future evaluation of the diagram in any state becomes more expensive
(since more substitutions need to be considered). However,
this is only true if the diagram remains unchanged after object
maximization. In fact, as illustrated in the example given below,
these variables may be pruned from the diagram in the process of
reduction. Thus as long as the final value function is compact the
evaluation is efficient and there is no such hidden cost.

\subsection{Maximizing Over Actions} 

The maximization 
$V_{n+1} = \max_A Q_{n+1}^A$
in step (3) combines independent functions. Therefore as above
we
must first standardize apart the
different diagrams, then we can follow with 
the propositional Apply procedure
and finally follow with weak and strong reductions.
This clearly maintains correctness for any concrete instantiation of
the state space.

\subsection{Order Over Argument Types}

We can now resume the discussion of
ordering of argument types and extend it to predicate and action parameters.
As above, some structure is suggested by the operations of the
algorithm. 
Section~\ref{Sec:labelorder} already suggested that we order constants
before variables.  

Action parameters are ``special constants'' before object maximization
but they become variables during object maximization. Thus their
position should allow them to behave as variables. We 
should therefore also order constants before action parameters.

Note that predicate parameters only exist inside TVDs, and will be
replaced with domain constants or variables during regression.
Thus we only need to decide on
the relative order between predicate
parameters and action parameters.  
If we put action parameters before predicate parameters and the latter
is replaced with a constant then we get an order violation, so this
order is not useful.
On the other hand, 
if we put predicate parameters before action parameters then both
instantiations of predicate parameters are possible.
Notice that when substituting a predicate parameter 
with a variable, action parameters still need to be larger 
than the variable (as they were in the TVD).
Therefore, we also order action parameters after variables.

To summarize, 
the ordering:
constants $\prec$ variables (predicate
parameters in case of TVDs) $\prec$ action parameters,
is suggested by 
heuristic considerations for orders that maximize the
potential for reductions, and avoid the need for re-sorting diagrams.

Finally, note that if we want to maintain the diagram sorted at all
times, we need to maintain variant versions of each TVD capturing
possible ordering of replacements of predicate parameters.
Consider a TVD in Figure~\ref{Fig:labelOrdering}(a). 
If we rename predicate parameters $X$ and $Y$ to be
$x_2$ and $x_1$ respectively, and if $x_1 \prec x_2$, then the resulting
sub-FODD as shown in Figure~\ref{Fig:labelOrdering}(b)
violates the order. To solve this problem we have to
define another TVD 
corresponding to the case where the substitution of $X$
$\succ$ the substitution of $Y$,
as shown in   Figure~\ref{Fig:labelOrdering}(c).
In the case of replacing $X$ with $x_2$ and $Y$ with
$x_1$, we use the TVD in  Figure~\ref{Fig:labelOrdering}(c)
instead of the one in Figure~\ref{Fig:labelOrdering}(a).

\begin{figure}[tbhp]
\vskip 0.2in
\begin{center}
\centerline{\psfig{figure=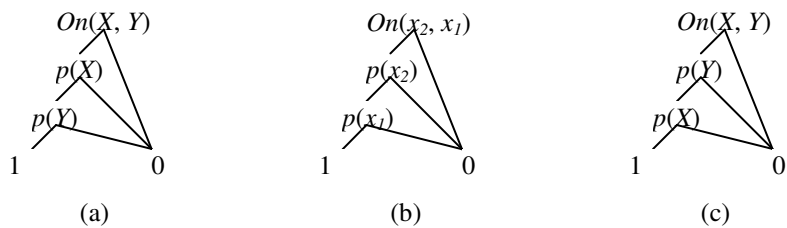}}
\caption{An example illustrating the necessity to maintain multiple TVDs.}
\label{Fig:labelOrdering}
\end{center}
\vskip -0.2in
\end{figure}

\subsection{Convergence and Complexity}

Since each step of Procedure~\ref{Proc:relgreedy}  is correct we have the following theorem:
\begin{theorem}
\label{Thm:VI}
Consider any concrete instantiation of a RMDP. Let
$V_n$ be 
the value function 
for the corresponding MDP 
when there are $n$ steps to go.
Then the value of $V_{n+1}$ calculated by 
Procedure~\ref{Proc:relgreedy}
correctly captures the value function when there 
are $n+1$ steps to go.
That is, for any state $s$, $\map_{V_{n+1}}(s)$ is 
the maximum expected value achievable in $s$ in $n+1$ steps.
\end{theorem}

Note that for RMDPs some problems require an infinite 
number of state partitions. Thus we cannot converge to $V^*$ in a finite
number of steps. However, since our algorithm implements 
VI exactly, standard results about approximating optimal
value functions and policies still hold. In particular the following
standard result~\cite{Puterman1994} holds for our algorithm, and
our stopping criterion guarantees approximating optimal value functions
and policies. 

\begin{theorem}
\label{Thm:VIconvergence}
Let $V^*$ be the optimal value function and let $V_k$ be the 
value function calculated by the relational VI algorithm.\\
(1) If $r(s)\leq M$ for all $s$ then $\|V_n-V^*\|\leq \varepsilon$
for $n\geq \frac{log(\frac{2M}{\varepsilon(1-\gamma)})}{log \frac{1}{\gamma}}$.\\
(2) If $\|V_{n+1}-V_n\|\leq \frac{\varepsilon(1-\gamma)}{2\gamma}$
then $\|V_{n+1}-V^*\|\leq \varepsilon$. 
\end{theorem}

While the algorithm maintains compact diagrams, reduction of diagrams
is not guaranteed for all domains. Therefore we can only
provide trivial upper bounds in terms of worst case time complexity.
Notice first 
that every time we use the Apply procedure the size of the output
diagram may be as large as the product of the size of its inputs.
We must also consider the size of the FODD giving
the regressed value function.
While Block replacement is $O(N)$ where $N$ is the size of 
the current value function, it is not sorted and sorting
may require both exponential time and space in the worst case.
For example, \citeA{Bryant1986} illustrates how ordering
may affect the size of a diagram. For a function of $2n$
 arguments, the function 
$x_1\cdot x_2 + x_3\cdot x_4 + \cdots +x_{2n-1}\cdot x_{2n}$ 
only requires a diagram
of $2n+2$ nodes, while the function 
$x_1\cdot x_{n+1} + x_2\cdot x_{n+2}+ \cdots +x_n\cdot x_{2n}$
requires $2^{n+1}$ nodes.
Notice that these two functions only differ by a permutation
of their arguments.
Now if $x_1\cdot x_2 + x_3\cdot x_4 + \cdots +x_{2n-1}\cdot x_{2n}$
is the result of block replacement then clearly sorting
requires exponential time and space. The same is 
true for our block combination procedure or any other method of
calculating the result, simply because the output
is of exponential size. In such a case heuristics that change variable
ordering, as in propositional ADDs \cite{Bryant1992}, would probably
be very useful.

Assuming TVDs, reward function, and probabilities
all have size $\leq C$,  each action has 
$\leq M$ action alternatives, 
the current value function $V_n$ has $N$ nodes, and
worst case space  expansion for regression and all Apply operations,
the overall size of the result and the time complexity for  
one iteration are  $O(C^{M^2(N+1)})$.
However note that this is the worst case analysis and does not take reductions
into account. 
While our method is not guaranteed to always work efficiently,
the alternative of grounding the MDP will have an unmanageable number
of states to deal with, so despite the high worst case
complexity our method provides a potential improvement.
As the next example illustrates, reductions can 
substantially decrease diagram size and therefore
save considerable time in computation.

\subsection{A Comprehensive Example of Value Iteration }
\label{Sec:VIexample}

Figure~\ref{Fig:regression}
traces steps in the application of value iteration to the
logistics domain.
The TVDs, action choice probabilities, and reward function
for this domain
are given in Figure~\ref{Fig:logisticsFODD}.
To simplify the presentation,
we continue using the predicate ordering 
$Bin \prec $  ``=''  $\prec On \prec Tin \prec rain$
introduced earlier.\footnote{The 
details do not change substantially if we use
the order suggested in Section~\ref{Sec:labelorder} (where equality is first).
}

\begin{figure}[tbhp]
\centering
\centerline{\psfig{figure=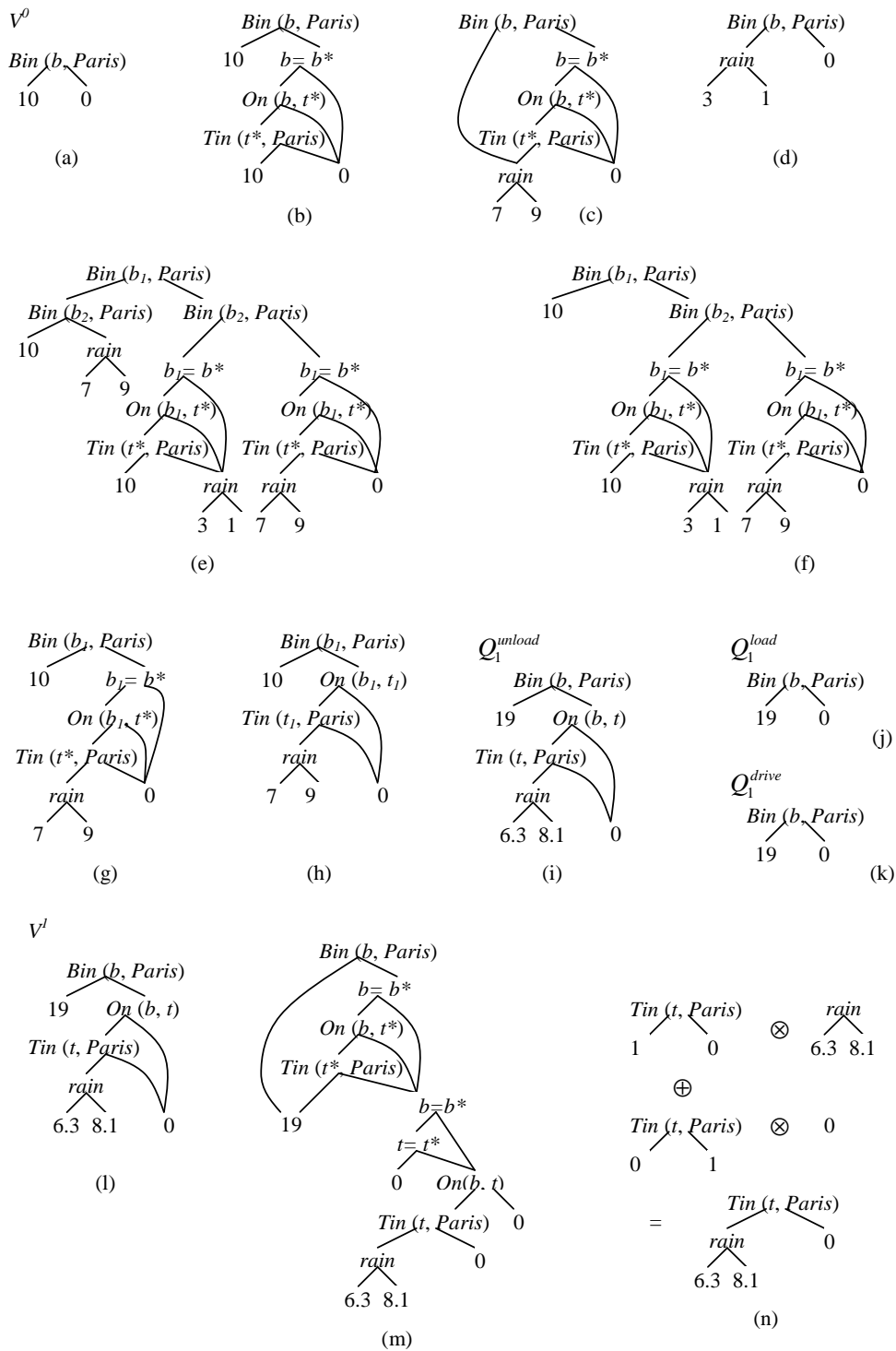,width=5.6in}}
\caption{An example of value iteration in the Logistics Domain.}
\label{Fig:regression}
\end{figure}

Given $V_0=R$ as shown in Figure~\ref{Fig:regression}(a),
{\bf Figure~\ref{Fig:regression}(b)} gives the result of regression 
of $V_0$ through $unloadS(b^*, t^*)$ by block replacement,
denoted as $Regr(V_0, unloadS(b^*,t^*))$.

{\bf Figure~\ref{Fig:regression}(c)} gives the result of multiplying
 $Regr(V_0, unloadS(b^*,t^*))$
with the choice probability of $unloadS$ $Pr(unloadS(b^*,t^*))$.

{\bf Figure~\ref{Fig:regression}(d)} gives the result of
$Pr(unloadF(b^*,t^*))\otimes Regr(V_0, unloadF(b^*,t^*))$.
Notice that this diagram is simpler since $unloadF$ does not
change the state and the TVDs for it are trivial.

{\bf Figure~\ref{Fig:regression}(e)} gives the   
unreduced result of adding two outcomes for $unload(b^*, t^*)$,
i.e., the result of adding 
$[Pr(unloadS(b^*,t^*))\otimes Regr(V_0, unloadS(b^*,t^*))]$ to
$[Pr(unloadF(b^*,t^*))$ $\otimes Regr(V_0, unloadF(b^*,t^*))].$
Note that we first  standardize apart diagrams for $unloadS(b^*,t^*)$
 and $unloadF(b^*,t^*)$ by respectively
renaming $b$ as $b_1$ and $b_2$.
 Action parameters $b^*$ and $t^*$ at this stage are
considered as constants and we do not change them.
Also note that the recursive part of Apply (addition $\oplus$) has performed some reductions, i.e.,
removing the node $rain$ when both of its children lead to value $10$.

In Figure~\ref{Fig:regression}(e),
we can apply R6 to node $Bin(b_2, Paris)$ in the left branch.
The conditions \\
 P7.1: $[\exists b_1,Bin(b_1,Paris)]\rightarrow [\exists b_1,b_2, Bin(b_1,Paris)\wedge Bin(b_2,Paris)]$,\\
 V7.1: $min(Bin(b_2,Paris)\downt)=10 \geq max(Bin(b_2,Paris)\downf)=9$, \\
 V7.2: $Bin(b_2, Paris)\downt$ is a constant \\
hold.
According to Lemma~\ref{Lemma:R7replace1} and Lemma~\ref{Lemma:R7drop1}  we can drop node $Bin(b_2,Paris)$ and connect its parent $Bin(b_1, Paris)$
to its true branch.
{\bf Figure~\ref{Fig:regression}(f)} gives the result after this reduction.

Next, consider the \true child of $Bin(b_2,Paris)$ and
the \true child of the root.
The conditions \\
P7.1: $[\exists b_1,b_2, \neg Bin(b_1, Paris)\wedge Bin(b_2,Paris)]
\rightarrow [\exists b_1, Bin(b_1,Paris)]$,\\
V7.1: $min(Bin(b_1,Paris)\downt)=10 \geq max(Bin(b_2,Paris)\downt)=10$, \\
V7.2: $min(Bin(b_1,Paris)\downt)=10 \geq max(Bin(b_2,Paris)\downf)=9 $ \\
hold.
According to Lemma~\ref{Lemma:R7replace1} and Lemma~\ref{Lemma:R7drop1}, we can drop the node $Bin(b_2, Paris)$ and connect
its parent $Bin(b_1, Paris)$ to  $Bin(b_2, Paris)\downf$.
{\bf Figure~\ref{Fig:regression}(g)}  gives the result after this reduction and now
we get a fully reduced diagram.
This is $T_{V_0}^{unload(b^*,t^*)}$.

In the next step we perform object maximization to maximize over action
parameters $b^*$ and $t^*$ and get the best instance of the action
$unload$. Note that $b^*$ and $t^*$  have now become variables, and we
can perform one more reduction: we can drop the equality on the right branch by 
R9.
{\bf Figure~\ref{Fig:regression}(h)} gives the result after object maximization,
i.e., $\mbox{obj-max} (T_{V_0}^{unload(b^*,t^*)})$.
Note that we have renamed the action parameters to avoid the repetition
between iterations.

{\bf Figure~\ref{Fig:regression}(i)} gives the
reduced result of multiplying 
Figure~\ref{Fig:regression}(h), $\mbox{obj-max} (T_{V_0}^{unload(b^*,t^*)})$, 
by $\gamma=0.9$, and adding  the reward function.
This result is $Q^{unload}_1$.

We can calculate $Q^{load}_1$ and $Q^{drive}_1$
in the same way and results are shown 
in {\bf Figure~\ref{Fig:regression}(j)} and 
{\bf Figure~\ref{Fig:regression}(k)} respectively.
For $drive$ the TVDs are trivial and the calculation is relatively 
simple. For $load$, the potential loading of 
a box already in Paris is dropped from the diagram by the reduction operators
in the process of object maximization.

{\bf Figure~\ref{Fig:regression}(l)} gives $V_1$, the result
after maximizing over $Q^{unload}_1$, $Q^{load}_1$ and $Q^{drive}_1$.
Here again we standardized apart the diagrams,
maximized over them, and then reduced the result.   
In this case the diagram for $unload$ dominates
  the other actions. Therefore $Q^{unload}_1$ becomes $V_1$, the
value function after the first iteration.

Now we can 
start the second iteration, i.e., computing $V_2$ from $V_1$.
{\bf Figure~\ref{Fig:regression}(m)}
gives the result of
block replacement in regression of $V^1$ through action alternative $unloadS(b^*, t^*)$.
Note that we have sorted the TVD for $on(B,T)$ so that it obeys the
ordering we have chosen. 
However, the diagram resulting from block replacement
is not sorted.

To address this we  use the block combination algorithm 
to combine blocks bottom up. 
{\bf Figure~\ref{Fig:regression}(n)} illustrates how we combine
blocks $Tin(t,Paris)$, which is a TVD, and its two children, which have 
been processed and are general FODDs.
After we combine $Tin(t,Paris)$ and its two children, $On(b,t)\downt$ has been processed.
Since  $On(b,t)\downf=0$,
now we can combine
 $On(b,t)$ and its two children in the next step of block combination.
Continuing this process we get a sorted representation of 
$Regr(V_1, unloadS(b^*,t^*))$.

\subsection{Extracting Optimal Policies}
 There is more than one way to represent policies with FODDs. 
Here we simply note that a policy can be represented implicitly by a set of regressed value functions.
After the value iteration terminates, we can 
perform one more iteration and compute the set of $Q$-functions using 
Equation~\ref{Eq:Qfunction}.

Then, given a state $s$, we can compute the maximizing action as follows:
\begin{enumerate}
\item
For each $Q$-function $Q^{A(\vec{x})}$, 
compute $\map_{Q^{A(\vec{x})}}(s)$, where $\vec{x}$ are considered as variables. 
\item
For the maximum map obtained, record the action name and action parameters (from the 
valuation) to obtain the maximizing action.
\end{enumerate}

This clearly implements the policy represented by the value function.
An alternative approach that represents the policy explicitly
was developed in the context of a policy iteration algorithm
\cite{WangKh2007}.

\section{Discussion}

ADDs have been used successfully to solve propositional factored MDPs.
Our work gives one proposal of lifting these ideas to RMDPs. While
the general steps are similar, the technical details are significantly
more involved than the propositional case.  
Our decision diagram representation combines the strong points
of the SDP and ReBel approaches to RMDP. 
On the one hand we get simple regression
algorithms directly manipulating the diagrams. On the other hand we
get object maximization for free as in ReBel. 
We also get space saving
since 
different state partitions can share structure in the diagrams. 
A possible disadvantage compared to ReBel is that the reasoning required 
for reduction operators might be complex.

In terms of expressiveness, our approach can easily capture
probabilistic STRIPS style formulations as in ReBel, 
allowing for more flexibility
since we can use FODDs to capture rewards and transitions.
For example, our representation can capture universal effects of
actions. 
On the other hand, it is more limited than SDP
since we cannot use arbitrary formulas for rewards, transitions, and
probabilistic choice. For example we cannot express universal
quantification using maximum aggregation, so these cannot be used in
reward functions or in action preconditions.
Our approach can also   
capture grid-world RL
domains with state based reward (which are propositional) in factored
form since the reward can be described as a function of location. 

By contrasting the single path semantics with the multiple
path semantics we see an interesting tension between the choice of
representation and task. The multiple path method 
does not directly support state partitions, which 
makes it awkward to specify distributions and policies 
(since values and actions must both be specified at leaves). 
However, this semantics simplifies many steps by 
easily supporting disjunction and maximization over valuations
which are crucial for
for value iteration so it is likely to lead to
significant saving in space and time.

An implementation and empirical evaluation are in progress.
The precise choice of reduction operators and their application will
be crucial to obtain an effective system,
since in general there is a tradeoff between run time needed for reductions
and   the size of resulting FODDs.
We can apply complex reduction operators to get the maximally
reduced FODDs, but it takes longer to perform the reasoning required.
This optimization is still an open issue both theoretically and
empirically. 
Additionally, our implementation can easily incorporate the idea of
approximation by combining leaves with similar values 
to control the size of FODDs
\cite{St-AubinHoBo2000}. This gives a simple way of trading
off efficiency against accuracy of the value functions.

There are many open issues 
concerning the
current representation.
Our results for FODDs give a first step toward a complete
generalization of ADDs. Crucially we 
do not yet have a semantically appropriate normal form that is important
in simplifying reasoning. While one can
define a normal form 
\cite<cf.,>[for a treatment of conjunctions]{GarrigaKhRa2007}
it is not clear if this can be calculated incrementally 
using local operations as in ADDs.
It would be interesting to investigate conditions that guarantee a
normal form for a useful set of reduction operators for FODDs.

Another possible improvement is that the representation can be modified to allow further
compression. 
For example 
we can  allow edges to
rename variables when they are traversed so as to compress
isomorphic sub-FODDs as illustrated above in Figure~\ref{Fig:union}(c). 
Another interesting possibility is a copy
operator that evaluates several copies of a predicate
(with different variables) in the same node as illustrated in  Figure~\ref{Fig:copyop}.
For such constructs to be usable one must modify the FODD and MDP algorithmic
steps to handle diagrams with the new syntactic notation.

\begin{figure}[tbhp]
\vskip 0.2in
\begin{center}
\centerline{\psfig{figure=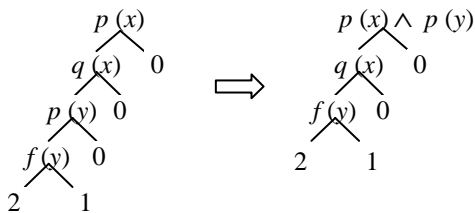}}
\caption{Example illustrating the copy operator.}
\label{Fig:copyop}
\end{center}
\vskip -0.2in
\end{figure} 

\section{Conclusion}
The paper makes two main contributions. 
First, we introduce FODDs, 
a generalization of ADDs, for relational domains that may be
useful in various applications.
We have developed calculus of FODDs and reduction operators 
to minimize their size but there are many open issues regarding 
the best choice of operators and reductions.
The second contribution is in developing 
a FODD-based value iteration algorithm for RMDPs
that has the potential for significant improvement
over previous approaches.
The algorithm performs general 
relational probabilistic reasoning without ever grounding the 
domains and it 
is proved to converge
to the abstract optimal value function
when such a solution exists.

\bibliography{FOMDP}

\begin{thebibliography}{}

\bibitem[\protect\BCAY{Bahar, Frohm, Gaona, Hachtel, Macii, Pardo,\ \BBA\
  Somenzi}{Bahar et~al.}{1993}]{BaharFrGaHaMaPaSo1993}
Bahar, R.~I., Frohm, E.~A., Gaona, C.~M., Hachtel, G.~D., Macii, E., Pardo, A.,
  \BBA\ Somenzi, F. \BBOP1993\BBCP.
\newblock \BBOQ Algebraic decision diagrams and their applications\BBCQ\
\newblock In {\Bem Proceedings of the International Conference on
  Computer-Aided Design}, \BPGS\ 188--191.

\bibitem[\protect\BCAY{Bellman}{Bellman}{1957}]{Bellman1957}
Bellman, R.~E. \BBOP1957\BBCP.
\newblock {\Bem Dynamic programming}.
\newblock Princeton University Press.

\bibitem[\protect\BCAY{Blockeel\ \BBA\ De~Raedt}{Blockeel\ \BBA\
  De~Raedt}{1998}]{BlockeelDR98}
Blockeel, H.\BBACOMMA\  \BBA\ De~Raedt, L. \BBOP1998\BBCP.
\newblock \BBOQ Top down induction of first order logical decision trees\BBCQ\
\newblock {\Bem Artificial Intelligence}, {\Bem 101}, 285--297.

\bibitem[\protect\BCAY{Boutilier, Dean,\ \BBA\ Goldszmidt}{Boutilier
  et~al.}{2000}]{BoutilierDeGo2000}
Boutilier, C., Dean, T., \BBA\ Goldszmidt, M. \BBOP2000\BBCP.
\newblock \BBOQ Stochastic dynamic programming with factored
  representations\BBCQ\
\newblock {\Bem Artificial Intelligence}, {\Bem 121(1)}, 49--107.

\bibitem[\protect\BCAY{Boutilier, Dean,\ \BBA\ Hanks}{Boutilier
  et~al.}{1999}]{BoutilierDeHa1999}
Boutilier, C., Dean, T., \BBA\ Hanks, S. \BBOP1999\BBCP.
\newblock \BBOQ Decision-theoretic planning: Structural assumptions and
  computational leverage\BBCQ\
\newblock {\Bem Journal of Artificial Intelligence Research}, {\Bem 11}, 1--94.

\bibitem[\protect\BCAY{Boutilier, Dearden,\ \BBA\ Goldszmidt}{Boutilier
  et~al.}{1995}]{BoutilierDeMo1995}
Boutilier, C., Dearden, R., \BBA\ Goldszmidt, M. \BBOP1995\BBCP.
\newblock \BBOQ Exploiting structure in policy construction\BBCQ\
\newblock In {\Bem Proceedings of the International Joint Conference of
  Artificial Intelligence}, \BPGS\ 1104--1111.

\bibitem[\protect\BCAY{Boutilier, Reiter,\ \BBA\ Price}{Boutilier
  et~al.}{2001}]{BoutilierRePr2001}
Boutilier, C., Reiter, R., \BBA\ Price, B. \BBOP2001\BBCP.
\newblock \BBOQ Symbolic dynamic programming for first-order {MDP}'s\BBCQ\
\newblock In {\Bem Proceedings of the International Joint Conference of
  Artificial Intelligence}, \BPGS\ 690--700.

\bibitem[\protect\BCAY{Bryant}{Bryant}{1986}]{Bryant1986}
Bryant, R.~E. \BBOP1986\BBCP.
\newblock \BBOQ Graph-based algorithms for boolean function manipulation\BBCQ\
\newblock {\Bem IEEE Transactions on Computers}, {\Bem C-35\/}(8), 677--691.

\bibitem[\protect\BCAY{Bryant}{Bryant}{1992}]{Bryant1992}
Bryant, R.~E. \BBOP1992\BBCP.
\newblock \BBOQ Symbolic boolean manipulation with ordered binary decision
  diagrams\BBCQ\
\newblock {\Bem ACM Computing Surveys}, {\Bem 24\/}(3), 293--318.

\bibitem[\protect\BCAY{Cormen, Leiserson, Rivest,\ \BBA\ Stein}{Cormen
  et~al.}{2001}]{CormenLeRiSt2001}
Cormen, T.~H., Leiserson, C.~E., Rivest, R.~L., \BBA\ Stein, C. \BBOP2001\BBCP.
\newblock {\Bem Introduction to Algorithms}.
\newblock MIT Press.

\bibitem[\protect\BCAY{Driessens, Ramon,\ \BBA\ G\"{a}rtner}{Driessens
  et~al.}{2006}]{Driessens2006}
Driessens, K., Ramon, J., \BBA\ G\"{a}rtner, T. \BBOP2006\BBCP.
\newblock \BBOQ Graph kernels and gaussian processes for relational
  reinforcement learning\BBCQ\
\newblock {\Bem Machine Learning}, {\Bem 64\/}(1-3), 91--119.

\bibitem[\protect\BCAY{Dzeroski, De~Raedt,\ \BBA\ Driessens}{Dzeroski
  et~al.}{2001}]{DzeroskiDeDr01}
Dzeroski, S., De~Raedt, L., \BBA\ Driessens, K. \BBOP2001\BBCP.
\newblock \BBOQ Relational reinforcement learning\BBCQ\
\newblock {\Bem Machine Learning}, {\Bem 43}, 7--52.

\bibitem[\protect\BCAY{Feng\ \BBA\ Hansen}{Feng\ \BBA\
  Hansen}{2002}]{FengHa2002}
Feng, Z.\BBACOMMA\  \BBA\ Hansen, E.~A. \BBOP2002\BBCP.
\newblock \BBOQ Symbolic heuristic search for factored {Markov} {D}ecision
  {P}rocesses\BBCQ\
\newblock In {\Bem Proceedings of the National Conference on Artificial
  Intelligence}, \BPGS\ 455--460.

\bibitem[\protect\BCAY{Fern, Yoon,\ \BBA\ Givan}{Fern
  et~al.}{2003}]{FernYoGi2003}
Fern, A., Yoon, S., \BBA\ Givan, R. \BBOP2003\BBCP.
\newblock \BBOQ Approximate policy iteration with a policy language bias\BBCQ\
\newblock In {\Bem International Conference on Neural Information Processing
  Systems}.

\bibitem[\protect\BCAY{Fern, Yoon,\ \BBA\ Givan}{Fern
  et~al.}{2006}]{FernYoGi2006}
Fern, A., Yoon, S., \BBA\ Givan, R. \BBOP2006\BBCP.
\newblock \BBOQ Approximate policy iteration with a policy language bias:
  Solving relational {Markov Decision Processes}\BBCQ\
\newblock {\Bem Journal of Artificial Intelligence Research}, {\Bem 25},
  75--118.

\bibitem[\protect\BCAY{Garriga, Khardon,\ \BBA\ De~Raedt}{Garriga
  et~al.}{2007}]{GarrigaKhRa2007}
Garriga, G., Khardon, R., \BBA\ De~Raedt, L. \BBOP2007\BBCP.
\newblock \BBOQ On mining closed sets in multi-relational data\BBCQ\
\newblock In {\Bem Proceedings of the International Joint Conference of
  Artificial Intelligence}, \BPGS\ 804--809.

\bibitem[\protect\BCAY{Gretton\ \BBA\ Thiebaux}{Gretton\ \BBA\
  Thiebaux}{2004}]{GrettonTh2004}
Gretton, C.\BBACOMMA\  \BBA\ Thiebaux, S. \BBOP2004\BBCP.
\newblock \BBOQ Exploiting first-order regression in inductive policy
  selection\BBCQ\
\newblock In {\Bem Proceedings of the Conference on Uncertainty in Artificial
  Intelligence}, \BPGS\ 217--225.

\bibitem[\protect\BCAY{Groote\ \BBA\ Tveretina}{Groote\ \BBA\
  Tveretina}{2003}]{GrooteTv2003}
Groote, J.~F.\BBACOMMA\  \BBA\ Tveretina, O. \BBOP2003\BBCP.
\newblock \BBOQ Binary decision diagrams for first-order predicate logic\BBCQ\
\newblock {\Bem The Journal of Logic and Algebraic Programming}, {\Bem 57},
  1--22.

\bibitem[\protect\BCAY{Gro{\ss}mann, H{\"o}lldobler,\ \BBA\
  Skvortsova}{Gro{\ss}mann et~al.}{2002}]{GroBmannHoSk2002}
Gro{\ss}mann, A., H{\"o}lldobler, S., \BBA\ Skvortsova, O. \BBOP2002\BBCP.
\newblock \BBOQ Symbolic dynamic programming within the fluent calculus\BBCQ\
\newblock In {\Bem Proceedings of the IASTED International Conference on
  Artificial and Computational Intelligence}.

\bibitem[\protect\BCAY{Guestrin, Koller, Gearhart,\ \BBA\ Kanodia}{Guestrin
  et~al.}{2003a}]{GuestrinKoGeKa2003}
Guestrin, C., Koller, D., Gearhart, C., \BBA\ Kanodia, N. \BBOP2003a\BBCP.
\newblock \BBOQ Generalizing plans to new environments in relational
  {MDPs}\BBCQ\
\newblock In {\Bem Proceedings of the International Joint Conference of
  Artificial Intelligence}, \BPGS\ 1003--1010.

\bibitem[\protect\BCAY{Guestrin, Koller, Par,\ \BBA\ Venktaraman}{Guestrin
  et~al.}{2003b}]{GuestrinKoPaVe2003}
Guestrin, C., Koller, D., Par, R., \BBA\ Venktaraman, S. \BBOP2003b\BBCP.
\newblock \BBOQ Efficient solution algorithms for factored {MDPs}\BBCQ\
\newblock {\Bem Journal of Artificial Intelligence Research}, {\Bem 19},
  399--468.

\bibitem[\protect\BCAY{Hansen\ \BBA\ Feng}{Hansen\ \BBA\
  Feng}{2000}]{HansenFe2000}
Hansen, E.~A.\BBACOMMA\  \BBA\ Feng, Z. \BBOP2000\BBCP.
\newblock \BBOQ Dynamic programming for {POMDPs} using a factored state
  representation\BBCQ\
\newblock In {\Bem Proceedings of the International Conference on Artificial
  Intelligence Planning Systems}, \BPGS\ 130--139.

\bibitem[\protect\BCAY{Hoey, St-Aubin, Hu,\ \BBA\ Boutilier}{Hoey
  et~al.}{1999}]{HoeyStHuBo1999}
Hoey, J., St-Aubin, R., Hu, A., \BBA\ Boutilier, C. \BBOP1999\BBCP.
\newblock \BBOQ {SPUDD}: Stochastic planning using decision diagrams\BBCQ\
\newblock In {\Bem Proceedings of the Conference on Uncertainty in Artificial
  Intelligence}, \BPGS\ 279--288.

\bibitem[\protect\BCAY{H{\"o}olldobler, Karabaev,\ \BBA\
  Skvortsova}{H{\"o}olldobler et~al.}{2006}]{HolldoblerKaSk2006}
H{\"o}olldobler, S., Karabaev, E., \BBA\ Skvortsova, O. \BBOP2006\BBCP.
\newblock \BBOQ {FluCaP:} a heuristic search planner for first-order
  {MDPs}\BBCQ\
\newblock {\Bem Journal of Artificial Intelligence Research}, {\Bem 27},
  419--439.

\bibitem[\protect\BCAY{Kersting, Otterlo,\ \BBA\ De~Raedt}{Kersting
  et~al.}{2004}]{KerstingOtRa2004}
Kersting, K., Otterlo, M.~V., \BBA\ De~Raedt, L. \BBOP2004\BBCP.
\newblock \BBOQ Bellman goes relational\BBCQ\
\newblock In {\Bem Proceedings of the International Conference on Machine
  Learning}.

\bibitem[\protect\BCAY{McMillan}{McMillan}{1993}]{McMillan1993}
McMillan, K.~L. \BBOP1993\BBCP.
\newblock {\Bem Symbolic model checking}.
\newblock Kluwer Academic Publishers.

\bibitem[\protect\BCAY{Puterman}{Puterman}{1994}]{Puterman1994}
Puterman, M.~L. \BBOP1994\BBCP.
\newblock {\Bem Markov decision processes: Discrete stochastic dynamic
  programming}.
\newblock Wiley.

\bibitem[\protect\BCAY{Rivest}{Rivest}{1987}]{Rivest1987}
Rivest, R.~L. \BBOP1987\BBCP.
\newblock \BBOQ Learning decision lists\BBCQ\
\newblock {\Bem Machine Learning}, {\Bem 2\/}(3), 229--246.

\bibitem[\protect\BCAY{Sanghai, Domingos,\ \BBA\ Weld}{Sanghai
  et~al.}{2005}]{SanghaiDoWe2005}
Sanghai, S., Domingos, P., \BBA\ Weld, D. \BBOP2005\BBCP.
\newblock \BBOQ Relational dynamic bayesian networks\BBCQ\
\newblock {\Bem Journal of Artificial Intelligence Research}, {\Bem 24},
  759--797.

\bibitem[\protect\BCAY{Sanner\ \BBA\ Boutilier}{Sanner\ \BBA\
  Boutilier}{2005}]{SannerBo2005}
Sanner, S.\BBACOMMA\  \BBA\ Boutilier, C. \BBOP2005\BBCP.
\newblock \BBOQ Approximate linear programming for first-order {MDPs}\BBCQ\
\newblock In {\Bem Proceedings of the Conference on Uncertainty in Artificial
  Intelligence}.

\bibitem[\protect\BCAY{Sanner\ \BBA\ Boutilier}{Sanner\ \BBA\
  Boutilier}{2006}]{SannerBo2006}
Sanner, S.\BBACOMMA\  \BBA\ Boutilier, C. \BBOP2006\BBCP.
\newblock \BBOQ Practical linear value-approximation techniques for first-order
  {MDPs}\BBCQ\
\newblock In {\Bem Proceedings of the Conference on Uncertainty in Artificial
  Intelligence}.

\bibitem[\protect\BCAY{Sanner\ \BBA\ Boutilier}{Sanner\ \BBA\
  Boutilier}{2007}]{SannerBo2007}
Sanner, S.\BBACOMMA\  \BBA\ Boutilier, C. \BBOP2007\BBCP.
\newblock \BBOQ Approximate solution techniques for factored first-order
  {MDPs}\BBCQ\
\newblock In {\Bem Proceedings of the International Conference on Automated
  Planning and Scheduling}.

\bibitem[\protect\BCAY{Schuurmans\ \BBA\ Patrascu}{Schuurmans\ \BBA\
  Patrascu}{2001}]{schuurmansPa2001}
Schuurmans, D.\BBACOMMA\  \BBA\ Patrascu, R. \BBOP2001\BBCP.
\newblock \BBOQ Direct value approximation for factored {MDPs}\BBCQ\
\newblock In {\Bem International Conference on Neural Information Processing
  Systems}, \BPGS\ 1579--1586.

\bibitem[\protect\BCAY{St-Aubin, Hoey,\ \BBA\ Boutilier}{St-Aubin
  et~al.}{2000}]{St-AubinHoBo2000}
St-Aubin, R., Hoey, J., \BBA\ Boutilier, C. \BBOP2000\BBCP.
\newblock \BBOQ {APRICODD}: Approximate policy construction using decision
  diagrams\BBCQ\
\newblock In {\Bem International Conference on Neural Information Processing
  Systems}, \BPGS\ 1089--1095.

\bibitem[\protect\BCAY{Wang}{Wang}{2007}]{Wang2007}
Wang, C. \BBOP2007\BBCP.
\newblock \BBOQ First order {Markov Decision Processes}\BBCQ\
\newblock \BTR\ TR-2007-4, Computer Science Department, Tufts University.

\bibitem[\protect\BCAY{Wang, Joshi,\ \BBA\ Khardon}{Wang
  et~al.}{2007}]{WangJoKh2007}
Wang, C., Joshi, S., \BBA\ Khardon, R. \BBOP2007\BBCP.
\newblock \BBOQ First order decision diagrams for relational {MDPs}\BBCQ\
\newblock In {\Bem Proceedings of the International Joint Conference of
  Artificial Intelligence}, \BPGS\ 1095--1100.

\bibitem[\protect\BCAY{Wang\ \BBA\ Khardon}{Wang\ \BBA\
  Khardon}{2007}]{WangKh2007}
Wang, C.\BBACOMMA\  \BBA\ Khardon, R. \BBOP2007\BBCP.
\newblock \BBOQ Policy iteration for relational {MDPs}\BBCQ\
\newblock In {\Bem Proceedings of the Conference on Uncertainty in Artificial
  Intelligence}.

\end{thebibliography}
\bibliographystyle{theapa}

\end{document}